\def\year{2016}\relax
\documentclass[letterpaper]{article}
\usepackage{aaai16}
\usepackage{times}
\usepackage{amsmath,amsthm}
\usepackage{enumitem}
\usepackage[latin9]{inputenc}
\usepackage{amstext}
\usepackage{amssymb}
\usepackage{multirow}
\usepackage{bbding}
\usepackage{textcomp}
\usepackage{tabularx}
\usepackage{graphicx}
\usepackage{xcolor}
\usepackage{tikz}
\usetikzlibrary{positioning,shapes,arrows,backgrounds}
\tikzset{
    events/.style={ellipse, draw, align=center},
}

\usepackage{float}
\floatstyle{ruled}
\newfloat{algorithm}{tbp}{loa}
\floatname{algorithm}{Algorithm}
\newtheorem {definition}{Definition}
\newtheorem {proposition}{Proposition}
\newtheorem {theorem}{Theorem}

\newtheorem {lemma}{Lemma}
\newtheorem {example}{Example}
\newtheorem {corollary}{Corollary}

\newenvironment{proof1} {\noindent\emph{Proof:}}{$\left.\right.$\hfill$\Box$}
\newcommand{\alias}[2]{\newcommand{#1}[0]{#2}}
\newcommand{\malias}[2]{\alias{#1}{\ensuremath{#2}}}

\malias{\q}{{\mathbb q}}
\malias{\Q}{{\mathcal Q}}
\alias{\yes}{{\sc yes}}
\alias{\no}{{\sc no}}
\alias{\PPc}{$PP$-$complete$}
\alias{\coNPc}{$coNP$-$complete$}
\alias{\NPc}{$NP$-$complete$}
\alias{\thetac}{$\Delta_2^P[O(log n)]$-$complete$}
\newcommand{\K}{{\mathcal{K}}}
\newcommand{\T}{{\mathcal{T}}}
\newcommand{\A}{{\mathcal{A}}}
\newcommand{\R}{{\mathcal{R}}}
\newcommand{\C}{{\mathcal{C}}}
\newcommand{\M}{{\mathcal{M}}}

\newcommand{\KB}[2]{\left<#1,#2\right>}
\newcommand{\dllite}{{{DL-Lite}}}

\newcommand{\eg}{\textit{e.g.}}
\newcommand{\resp}{\textit{resp.}}
\newcommand{\wrt}{w.r.t.}
\newcommand{\kb}{KB}

\newcommand{\Cl}{\textit{Cl}}
\newcommand{\cl}[1]{\circ_{cl}(#1)}
\newcommand{\lex}[1]{\circ_{card}(#1)}
\newcommand{\incl}[1]{\circ_{rep}(#1)}

\newcommand{\clf}{\circ_{cl}}
\newcommand{\lexf}{\circ_{card}}
\newcommand{\inclf}{\circ_{rep}}

\newcommand{\isa}{\sqsubseteq}
\newcommand{\isn}{\sqsubseteq$$\neg}

\newcommand{\tab}[1]{\begin{tabular}{c} #1 \end{tabular}}

\newcommand{\infr}[2]{\left<#1,#2\right>}
\newcommand{\infer}[2]{\models_{\langle \circ_#1, #2 \rangle}}

\newcommand{\ma}{\mathsf{R}}
\newcommand{\mb}{\mathsf{MR}}
\newcommand{\mc}{\mathsf{CMR}}
\newcommand{\md}{\mathsf{MCMR}}
\newcommand{\me}{\mathsf{CR}}
\newcommand{\mf}{\mathsf{MCR}}
\newcommand{\mg}{\mathsf{RC}}
\newcommand{\mh}{\mathsf{MRC}}
\newcommand{\m}{\circ}

\alias{\algPP}{{\sc MajSolver}}

\begin{document}
% The file aaai.sty is the style file for AAAI Press 
% proceedings, working notes, and technical reports.
\title{A General Modifier-based Framework\\ for Inconsistency-Tolerant Query Answering}
\author{Jean Fran\c{c}ois Baget \\
INRIA, France \\
baget@lirmm.fr\\
\And Salem Benferhat \\
Univ Artois, France\\
benferhat@cril.fr\\
\And 
Zied Bouraoui  \\
Univ Aix-Marseille, France\\
zied.bouraoui@lsis.fr\\
\And  Madalina Croitoru \\
Univ Montpellier, France \\
croitoru@lirmm.fr\\
\AND Marie-Laure Mugnier  \\
Univ Montpellier, France\\
mugnier@lirmm.fr\\
\And Odile Papini \\
Univ Aix-Marseille, France\\
odile.papini@amu-univ.fr\\
\And Swan Rocher \\
Univ Montpellier, France\\
rocher@lirmm.fr\\
\And  Karim Tabia \\
Univ Artois, France\\
tabia@cril.fr\\
}
\maketitle

%%%%%%%%%%%%%%%%%%%%%%%%%%%%%%%%%%%%%%%%%%%%%%%%%%%%%%%%%%%%%%%%%%%%%%%%%%%%%%%%%%%%%%%%%%%%%%%%%%%%%%%%%%%%%%%%%%%%%%%%%%%%%%%%%%%%%%%%%%%%%%%%%%%%%%%%%%%%%%%%%%%%%%%%%%%%%%%%%%%%%%%%%%%%%%%%%%%%%%%%%%%%%%%%%%%%%%%%%%%%%%%%%%%%

\begin{abstract}
We  propose a general framework for inconsistency-tolerant query answering within existential rule setting. This framework unifies the main semantics proposed by the state of art and introduces new ones based on cardinality and majority principles. It relies on two key notions: modifiers and inference strategies. An inconsistency-tolerant semantics is seen as a composite modifier plus an inference strategy. We compare the obtained semantics from a productivity point of view. 
\end{abstract}

%-----------------------------------------------------------------------------------------------------------------------------------------%
%																	             	        %
%								Section									                %
%																		                %
%-----------------------------------------------------------------------------------------------------------------------------------------%
\section{Introduction}
\label{sec:introduction}

In this paper we place ourselves in the context of Ontology-Based Data Access \cite{PoggiLCGLR08} and we address the problem of query answering when the assertional base (which stores data) is inconsistent with the ontology (which represents generic knowledge about a domain).  Existing work in this area studied different inconsistency-tolerant inference relations, called \emph{semantics}, which consist of getting rid of inconsistency by first computing a set of consistent subsets of the assertional base, called \emph{repairs}, that restore consistency w.r.t the ontology, then using them to perform query answering. Most of these proposals, inspired by database approaches \eg{} \cite{ArenasBC99} or propositional logic approaches \eg{} \cite{BenferhatDPW97}, were introduced for the lightweight description logic \dllite{} \eg{} \cite{lemboLRRS15}. Other description logics \eg{} \cite{Rosati11} or existential rule \eg{} \cite{LukasiewiczMPS15} have also been considered.  In this paper, we use existential rules \eg{}\cite{BagetLMS11} as ontology language that generalizes lightweight description logics.

The main contribution of this paper consists in setting up a general framework that unifies previous proposals and extends the state of the art with new semantics.  The idea behind our framework is to distinguish between the way data assertions are virtually distributed (notion of modifiers) and inference strategies.   An inconsistency-tolerant semantics is  then naturally defined by a modifier and an inference strategy. We also propose a classification of the productivity of hereby obtained semantics by sound and complete conditions relying on modifier inclusion and inference strategy order. The objective of framework is to establish a methodology for inconsistency handling which, by distinguishing between modifiers and strategies, allows not only to cover existing semantics, but also to easily define new ones, and to study different kinds of their properties.

%-----------------------------------------------------------------------------------------------------------------------------------------%
%																	             	        %
%								Section									                %
%																		                %
%-----------------------------------------------------------------------------------------------------------------------------------------%

\section{Preliminaries}
\label{sec:dllite}

We consider first-order logical languages without functional symbols, hence a \emph{term} is a variable or a constant. An \emph{atom} is of the form $p(t_1,\ldots,t_k)$ where $p$ is a predicate of arity $k$, and the $t_i$ are terms. Given an atom or a set of atoms $E$, \emph{terms}($E$) denotes the set of terms occurring in $E$. A (factual) \emph{assertion} is an atom without variables.

A \emph{conjunctive query} is an existentially quantified conjunction of atoms. For readability, we restrict our focus to \emph{Boolean} conjunctive queries, which are closed formulas. However the framework and the obtained results can be directly extended to general conjunctive queries. In the following, by \emph{query}, we mean a Boolean conjunctive query. Given a set of assertions $\A$ and a query $q$, the answer to $q$ over $\A$ is yes iff $\A \models q$, where $\models$ denotes the standard logical consequence.

A knowledge base can be seen as a database enhanced with an ontological component. Since inconsistency-tolerant query answering has been mostly studied in the context of description logics (DLs), and especially \dllite{}, we will use some DL vocabulary, like ABox for the data and TBox for the ontology.  However, our framework is not restricted to DLs, hence we define TBoxes and ABoxes in terms of first-order logic. We assume the reader familiar with the basics of DLs and their logical translation.

An \emph{ABox} is a set of factual assertions.  As a special case we have DL assertions restricted to unary and binary predicates.  A \emph{positive axiom} is of the form $\forall \mathbf{x} \forall \mathbf{y}(B[\mathbf{x},\mathbf{y}] \rightarrow \exists \mathbf{z}~H[\mathbf{y},\mathbf{z}])$ where $B$  and $H$ are conjunctions of atoms  (in other words, it is a positive existential rule).  As a special case, we have for instance concept and role inclusions in  \dllite$_R$,  which are respectively of the form  $B_1 \sqsubseteq B_2$  and $S_1 \sqsubseteq S_2$, where $B_i := A~| ~\exists S$  and $S_i := P ~| ~P^-$ (with $A$ an atomic concept, $P$ an atomic role and $P^{-}$ the inverse of an atomic role).  A \emph{negative axiom} is of the form $\forall \mathbf{x} (B[\mathbf{x}] \rightarrow \bot)$ where $B$ is a conjunction of atoms (in other words, it is a negative constraint). As a special case, we have for instance disjointness axioms in  \dllite$_R$, which are inclusions of the form $B_1\sqsubseteq\neg B_2$  and $S_1 \sqsubseteq \neg S_2$, or equivalently $B_1\sqcap B_2 \sqsubseteq \bot$  and $S_1\sqcap S_2 \sqsubseteq \bot$.

A \emph{TBox} $\T= \T_p \cup \T_n$ is partitioned into a set $\T_p$ of positive axioms and a set $\T_n$ of negative axioms.
Finally, a \emph{knowledge base} (\kb)  is of the form  $\K=\KB{\T}{\A}$ where $\A$ is an ABox and $\T$ is a TBox.  $\K$ is said to be \emph{consistent} if $\T \cup \A$ is satisfiable, otherwise it is said to be \emph{inconsistent}. We also say that $\A$ is (in)consistent (with $\T$), which reflects the assumption that the TBox is reliable. The answer to a query $q$ over a consistent \kb{} $\K$ is yes iff $\KB{\T}{\A}\models q$.  When $\K$ is inconsistent, standard consequence is not appropriate since all queries would be positively answered.

A key notion in inconsistency-tolerant query answering is the one of a repair of the ABox \wrt{} the TBox. A repair is a subset of the ABox consistent with the TBox and inclusion-maximal for this property: $\R\subseteq$$\A$ is a \emph{repair} of $\A$ \wrt{} $\T$ if  i) $\KB{\T}{\R}$ is consistent, and ii) $\forall \R' \subseteq \mathcal A$, if  $\R\varsubsetneq \R'$ ($\R$ is strictly included in $\R'$) then $\KB{\T}{\R'}$ is inconsistent. 
We denote by $\R(\A)$ the set of $\A$'s repairs (for easier reading, we often leave $\mathcal T$ implicit in our notations). Note that $\R(\A) = \{\mathcal A\}$ iff $\mathcal A$ is consistent. 
 The most commonly considered semantics for inconsistency-tolerant query answering, inspired from previous work in databases, is the following: $q$ is said to be a \emph{consistent consequence} of $\K$ if it is a standard consequence of each repair of $\A$. Several variants of this semantics have been proposed, which differ with respect to their behaviour (in particular they can be more or less cautious) and their computational complexity. Before recalling the main semantics studied in the literature, we need to introduce the notion of the positive closure of an ABox.
The \emph{positive closure} of $\A$ (w.r.t.{} $\T$), denoted by $\Cl(\A)$, is obtained by adding to $\A$ all assertions (built on the individuals occurring in $\A$) that can be inferred using the positive axioms of the TBox, namely: \\
$\Cl(\A)$=$\{A$ atom$ | \KB{\T_p}{\A}\models A$ and \emph{terms}($A$) $\subseteq$ \emph{terms}($\A$)$\}$
Note that the set of atomic consequences of a \kb{} $\K$=$ \KB{\T}{\A}$ may be infinite whereas the positive closure of $\A$ is always finite since it does not contain new terms. Note also that $\A$ is consistent (with $\T$) iff $\Cl(\A)$ is consistent (with $\T$).

We now recall the most well-known inconsistency-tolerant semantics introduced in \cite{ArenasBC99,LemboLRRS10,bienvenuAAAI12}. 
Given a  possibly inconsistent \kb{} $\K$=$\KB{\T}{\A}$, a query $q$ is said to be: 
\begin{itemize}
\item a consistent (or AR) consequence of $\K$ if $\forall \R\in \R(\A)$, \\ $\KB{\T}{\R}\models q$;
\item a CAR consequence of $\K$ if $\forall$$\R\in$$\R(\Cl(\A))$,$\KB{\T}{\R}$$\models$$q$;
\item an IAR consequence of $\K$ if $\KB{\T}{\bigcap_{\R\in\R(\A)} \R}\models q$;
\item an ICAR consequence of $\K$ if $\KB{\T}{\bigcap_{\R\in\Cl(\A)}\R}\models q$;
\item an ICR consequence of $\K$ if $\KB{\T}{\bigcap_{\R\in\R(\A)}\Cl(\R)}\models q$.
\end{itemize}

%-----------------------------------------------------------------------------------------------------------------------------------------%
%																	             	        %
%								Section									                %
%																		                %
%-----------------------------------------------------------------------------------------------------------------------------------------%

\section{A Unified Framework for Inconsistency-Tolerant Query Answering}
\label{sec:multiAbox}

In this section, we define a unified framework for inconsistency-tolerant query answering based on two main concepts:  modifiers and inference strategies.   

Let us first introduce the notion of MBox \kb{}s.  While a standard \kb{} has a single ABox, it is convenient for subsequent definitions to define \kb{}s with multiple ABoxes (``MBoxes''). Formally, an \emph{MBox  \kb{}} is of the form $\K_\M$=$\KB{\T}{\M}$ where $\T$ is a TBox and $\M$=$\{\A_1$,$\ldots$,$\A_n\}$ is a set of ABoxes called an MBox. 
We say that $\K_\M$ is \emph{consistent}, or  $\M$ is consistent (with $\T$) if each $\A_i$ in $\M$ is consistent (with $\T$).

In the following, we start with an MBox \kb{} which is a possibly inconsistent standard \kb{} (namely with a single ABox in $\M$) and produce a consistent MBox \kb{}, in which each element reflects a virtual reparation of the initial ABox.  We see an inconsistency-tolerant query answering method as made out of a \emph{modifier}, which produces a consistent MBox from the original ABox (and the Tbox), and an \emph{inference strategy}, which evaluates queries against the obtained MBox \kb{}.

%%%%%%%%%%%%%%%%%%%%%%%%%%%%%%%%%%%%%%%%%%%%%%%%%%%%%%%%%%%%%%%%%%%%%%%%%%%%%%%%%%%%%%%%%%%%%%%%%%%%%%%%%%%%%%%%%%%%%%%%%%%%%%%%%%%%%%%%%%%%%%%%%%%%%%%%%%%%%%%

\subsection{Elementary and Composite Modifiers}
\label{sbsec:mbox}

We first introduce three classes of elementary modifiers, namely expansion, splitting and selection.
For each class, we consider a "natural" instantiation, namely \emph{positive closure}, splitting into \emph{repairs} and selecting the largest elements (i.e., maximal \wrt{} \emph{cardinality}). 
 Elementary modifiers can be combined to define \emph{composite} modifiers.  Given the three natural instantiations of these modifiers, we show that their combination yields  exactly eight different composite modifiers. 

\paragraph{Expansion modifiers.}
The expansion of an MBox consists in  explicitly adding some inferred knowledge to its ABoxes. A natural expansion modifier consists in computing the \emph{positive closure} of an MBox, which is defined as follows:
$$\cl{\M}=\{\Cl(\A_i) | \A_i \in \M\}.$$

\paragraph{Splitting modifiers.} A splitting modifier replaces each $\A_i$ of an MBox by one or several of its consistent subsets. A natural splitting modifier consists of splitting each ABox into the set of its repairs, which is defined as follows:
$$\incl{\M}=\bigcup_{\A_i\in\M}\{\R(\A_i)\}.$$
This modifier always produces a consistent MBox.

\paragraph{Selection modifiers.} A selection modifier selects some subsets of an MBox. As a natural selection modifier, we consider the \emph{cardinality-based selection} modifier, which selects the largest elements of an MBox:
$$\lex{\M}=\{\A_i \in \M|\nexists \A_j \in \M\, s.t\,|\A_j| > |\A_i|\}.$$

\smallskip
We call a \emph{composite modifier} any combination of these three elementary modifiers. We now study the question of how many different composite modifiers yielding consistent MBoxes exist and how they compare to each other. We begin with some properties that considerably reduce the number of combinations to be considered. First, the three modifiers are idempotent. Second, the  modifiers $\clf$ and $\inclf$ need to be applied only once.

\begin{lemma}
\label{lem:equi-modif}
For any MBox $\M$, the following holds:
\begin{enumerate}
\item $\cl{\cl{\M}}$=$\cl{\M}$, $\incl{\incl{\M}}$=$\incl{\M}$ and  $\lex{\lex{\M}}$ =$\lex{\M}$.
\item Let $\circ_{d}$ be any composite modifier. Then: \\(a) $\cl{\circ_{d}(\cl{\M})}=\circ_{d}(\cl{\M})$, and \\
(b) $\incl{\circ_{d}(\incl{\M})}=\circ_{d}(\incl{\M})$.
\end{enumerate}
\end{lemma}

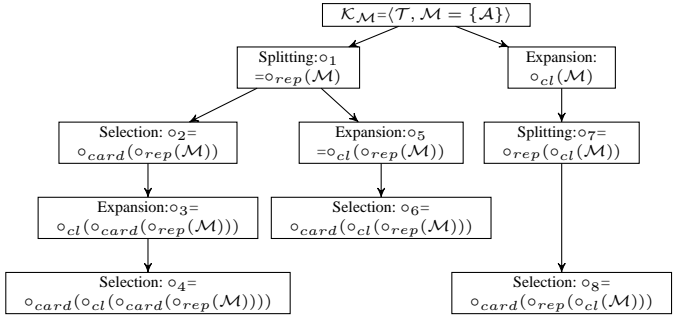
\begin{figure}[t]
\tiny
\begin{tikzpicture}[ >=stealth']
\tikzstyle{mbox}=[rectangle,draw, inner sep=0.8pt, outer sep=0,minimum size=0cm];
\node[mbox] (start) at (-0.3,0) {\tab{$\K_\M$=$\KB{\T}{\M = \{\A\}}$}};
%%%%%%%%%
\node[mbox] (C) at (1.5,-0.7) {\tab{Expansion: \\ $\cl{\M}$}};
\node[mbox] (CM) at  (1.5,-1.7) {\tab{Splitting:$\circ_7$=\\ $\incl{\cl{\M}}$}};
\node[mbox] (CML) at (1.5, -3.7) {\tab{Selection: $\circ_8$= \\ $\lex{\incl{\cl{\M}}}$}};
%%%%%%%%
\node[mbox] (M) at (-2,-0.7) {\tab{Splitting:$\circ_1$ \\{=$\incl{\M}$}}};
\node[mbox] (MC) at (-0.9, -1.7) {\tab{Expansion:$\circ_5$\\ =$\cl{\incl{\M}}$}};
\node[mbox] (MCL) at (-0.9,-2.7) {\tab{Selection: $\circ_6$=\\ $\lex{\cl{\incl{\M}}}$}};

\node[mbox] (ML) at (-4, -1.7) {\tab{Selection: $\circ_2$=\\$\lex{\incl{\M}}$}};
\node[mbox] (MLC) at (-4, -2.7) {\tab{Expansion:$\circ_3$= \\ $\cl{\lex{\incl{\M}}}$}};
\node[mbox] (MLCL) at (-4,-3.7) {\tab{Selection: $\circ_4$= \\ $\lex{\cl{\lex{\incl{\M}}}}$}};

\draw [->] (start) --  (C);
\draw [->] (C) --  (CM);
\draw [->] (CM) -- (CML);
\draw [->] (start) --  (M);
\draw [->] (M) --  (ML) ;
\draw [->] (M) --  (MC) ;
\draw [->] (ML) --  (MLC);
\draw [->] (MLC) --  (MLCL);
\draw [->] (MC) --  (MCL);
\end{tikzpicture}
\vspace{-0.5cm}
\caption{\small The eight possible combinations of modifiers starting from a single MBox \kb{} $\K_\M$=$\KB{\T}{\M=\{\A\}}$}
\label{sch:comp-modif}
\vspace{-0.2cm}
\end{figure}

Figure \ref{sch:comp-modif} presents the eight different composite modifiers (thanks to Lemma \ref{lem:equi-modif}) that can be applied to an MBox initially composed of a single (possibly inconsistent) ABox.
At the beginning, one can perform either  an expansion or a splitting operation (the selection has no effect). Expansion can only be followed  by a splitting or a selection operation. 
From the MBox $\incl{\cl{\M}}$ only a selection can be performed, thanks to Lemma \ref{lem:equi-modif}. %, where $\incl{\M}$  and $\cl{\M}$ only  need to be applied once. 
Similarly, if one starts with a splitting operation followed by a selection operation, then only an expansion can be  done (thanks to Lemma \ref{lem:equi-modif} again). % where $\incl{\M}$ needs only to be applied).  
From $\cl{\lex{\incl{\M}}}$ only a selection can be performed (Lemma \ref{lem:equi-modif} again). %, $\incl{\M}$  and $\cl{\M}$ needs only to be applied once.

To ease reading, we also denote the  modifiers by short names reflecting the order in which the elementary modifiers are applied, and using the following letters:  $\mathsf{R}$ for $\inclf$, $\mathsf{C}$ for $\clf$ and $\mathsf{M}$ for $\lexf$ as shown in Table \ref{tab:abox-modifs}. For instance, $\mf$ denotes the modifier  that first splits the initial ABox into its set of repairs, then closes these repairs and finally selects the maximal-cardinality elements.

\begin{table}
\centering\small
\begin{tabular}{|c|c|c|}
\hline Modifier & Combination& MBox \tabularnewline
\hline
$\ma$ & $\circ_1=\incl{.}$ & $\M_1=\circ_1(\M)$
\tabularnewline
\hline
$\mb$ & $\circ_2=\lex{\incl{.}}$ & $\M_2=\circ_2(\M)$
\tabularnewline
\hline
$\mc$ & $\circ_3=\cl{\lex{\incl{.}}}$ & $\M_3=\circ_3(\M)$
\tabularnewline
\hline
$\md$ & $\circ_4=\lex{\cl{\lex{\incl{.}}}}$ & $\M_4=\circ_4(\M)$
\tabularnewline
\hline
$\me$ & $\circ_5=\cl{\incl{.}}$   & $\M_5=\circ_5(\M)$
 \tabularnewline
\hline
$\mf$  & $\circ_6=\lex{\cl{\incl{.}}}$ & $\M_6=\circ_6(\M)$
\tabularnewline
\hline
 $\mg$  & $\circ_7=\incl{\cl{.}}$ & $\M_7=\circ_7(\M)$
\tabularnewline
\hline
$\mh$  & $\circ_8=\lex{\incl{\cl{.}}}$ & $\M_8=\circ_8(\M)$
\tabularnewline
\hline
\end{tabular}
\caption{The eight possible composite modifiers for an MBox $\K_\M$=$\KB{\T}{\M=\{\A\}}$}
\label{tab:abox-modifs}
\vspace{-0.5cm}
\end{table}

\begin{theorem}
\label{th:eight-modifier}
Let $\K_\M$=$\KB{\T}{\M=\{\A\}}$ be a possibly inconsistent  \kb{}. Then for any composite modifier $\circ_c$ that can be obtained by a finite combination of the elementary modifiers $\inclf$, $\lexf$, $\clf$, there exists a composite modifier $\circ_i$ in $\{\circ_1 \ldots \circ_8\}$ (see Table \ref{tab:abox-modifs}) such that $\circ_c(\M)$=$\circ_i(\M)$.
\end{theorem}

\begin{example}
\label{exp:mboxmodif}
Let $\K_\M$=$\KB{\T}{\M}$ be an MBox \dllite{} \kb{} where
$\T$=$\{A\isn B$, $A \isn C$, $B\isn C$, $A\isa D$, $B\isa D$, $C\isa D$, $B\isa E$, $C\isa E\}$ and
$\M$=$\{\{A(a),B(a),C(a),A(b)\}\}$. We have \\
$\circ_1(\M)$=$\{\{A(a),A(b)\}$,$\{B(a),A(b)\}$,$\{C(a),A(b)\}\}$,
$\circ_5(\M)$=$\{\{A(a),D(a),A(b),D(b)\}$,$\{B(a),D(a),E(a)$, $A(b)$,$D(b)\}$, $\{C(a),D(a),E(a),A(b),D(b)\}\}$, and \\
$\circ_6(\M)$=$\{\{B(a),D(a),E(a),A(b),D(b)\}, \{C(a),D(a)$, $E(a),A(b),D(b)\}\}$.
\end{example}

The composite modifiers can be classified according to "inclusion" as depicted in Figure \ref{sch:mbox-modif-relats}. We consider the relation, denoted by $\subseteq_R$, defined as follows:   given two modifiers $X$ and $Y$, $X$$\subseteq_R$$Y$ if, for any MBox $\M$, for each $A$$\in$$X(\M)$ there is $B$$\in$$Y(\M)$  such that $A \subseteq B$. We also consider two specializations of $\subseteq_R$:  the true inclusion $\subseteq$ (i.e., $X(\M) \subseteq Y(\M)$) and the "closure" inclusion, denoted by $\subseteq_{cl}$: $X$$\subseteq_{cl}$$Y$ if $Y(\M)$ is the positive closure of $X(\M)$ (then each  $A \in X(\M)$ is included in its closure in $Y(\M)$). In Figure \ref{sch:mbox-modif-relats}, there is an edge  from a modifier $X$ to a modifier $Y$ iff $X \subseteq_R Y$. We label each edge by the most specific inclusion relation that holds from $X$ to $Y$. Transitivity edges are not represented.

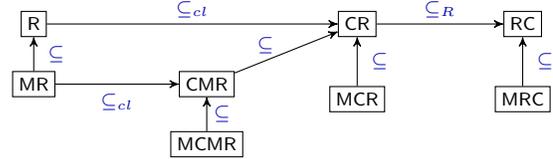
\begin{figure}[!h]
\centering\scriptsize
\begin{tikzpicture}[>=stealth']
\tikzstyle{mbox}=[rectangle,draw,fill=none];
\node[mbox] (m1) at (0,0)  {$\ma$};
\node[mbox] (m5) at (4.3,0)  {$\me$};
\node[mbox] (m7) at (6.5,0)  {$\mg$};
\node[mbox] (m2) at (0,-.8) {$\mb$};
\node[mbox] (m3) at (2.3,-.8) {$\mc$};
\node[mbox] (m4) at (2.3,-1.6) {$\md$};
\node[mbox] (m6) at (4.3,-1) {$\mf$};
\node[mbox] (m8) at (6.5,-1) {$\mh$};
\draw [->] (m1) -- (m5)node [above=0.1pt, pos=0.5,black!30!blue] {$\subseteq_{cl}$};
\draw [->] (m5) -- (m7) node [above=0.1pt, pos=0.5,black!30!blue] {$\subseteq_{R}$};
\draw [->] (m2) -- (m1) node [pos=0.5, right=0.1,black!30!blue] {$\subseteq$};
\draw [->] (m6) -- (m5) node [pos=0.5, right=0.1,black!30!blue] {$\subseteq$};
\draw [->] (m8) -- (m7) node [pos=0.5, right=0.1,black!30!blue] {$\subseteq$};
\draw [->] (m2) -- (m3) node [pos=0.5, below=1pt,black!30!blue] {$\subseteq_{cl}$};
\draw [->] (m4) -- (m3) node [pos=0.5, right=0.1pt, black!30!blue] {$\subseteq$};
\draw [->] (m3) -- (m5) node [pos=0.3, above=0.1pt,black!30!blue]{$\subseteq$};
\end{tikzpicture}
\vspace{-0.2cm}
\caption{Inclusion relations between composite modifiers.}
\label{sch:mbox-modif-relats}
\vspace{-0.3cm}
\end{figure}

\noindent With any $X$ and $Y$ such that $ X$$\subseteq_R$$Y$, one can naturally associate, for any MBox $\M$, a mapping from Mbox $X(\M)$ to MBox $Y(\M)$, which assigns each $A$$\in$$X(\M)$ to a $B$$\in$$Y(\M)$ such that $A$$\subseteq$$B$. We point out the following useful facts:
\begin{description}
\item [Fact 1] The MBox mapping associated with $\subseteq_R$ is injective in all our cases.
\item [Fact 2] The MBox mapping associated with $\subseteq_{cl}$ is surjective (hence bijective). The same holds for the mapping from $\me$ to $\mg$.
\end{description}

%%%%%%%%%%%%%%%%%%%%%%%%%%%%%%%%%%%%%%%%%%%%%%%%%%%%%%%%%%%%%%%%%%%%%%%%%%%%%%%%%%%%%%%%%%%%%%%%%%%%%%%%%%%%%%%%%%%%%%%%%%%%%%%%%%%%%%%%%%%%%%%%%%%%%%

\subsection{Inference Strategies for Querying an MBox}
\label{sbsec:multibox-inf}

An inference-based strategy takes as input a consistent MBox \kb{} $\K_\M$=$\KB{\T}{\M}$ and a query $q$ and determines if $q$ is entailed from $\K_\M$. We consider four main inference strategies: 
universal (also known as skeptical), safe, majority-based and existential (also called brave). \footnote{Of course, one can consider other inference strategies such as the argued inference, parametrized inferences, etc. This is left for future work. }

The \emph{universal} inference strategy states that a conclusion is valid iff it is entailed from $\T$ and every ABox in $\M$.  It is a standard way to derive conclusions from conflicting sources, used for instance in default reasoning \cite{reiter1980logic}, where one only accepts conclusions derived from each extension of a default theory.
 The \emph{safe} inference strategy considers as valid conclusions those  entailed from $\T$ and the intersection of all ABoxes.  The safe inference is a very sound and conservative inference relation since it only considers assertions shared by different ABoxes. The \emph{existential} inference strategy (called also brave inference relation) considers as valid all conclusions  entailed from $\T$ and at least one ABox.
 The existential inference is a very adventurous inference relation and may derive conclusions that are together inconsistent with $\T$. It is often considered as undesirable when the \kb{} represents available knowledge base on some problem. It only makes sense in some decision problems when one is only looking for a possible solution of a set of constraints or preferences. 
 Finally,  the \emph{majority-based} inference relation considers as valid all conclusions   entailed  from $\T$ and the majority of ABoxes. The majority-based inference can be seen as a good compromise between universal / safe inference  and existential inference.

We formally define these inference strategies as follows:
 \begin{itemize}
 \item Query $q$ is a \emph{universal} consequence of $\K_\M$, denoted by $\K_\M \models_{\forall} q$ iff $\forall \A_i \in \M$,$\KB{\T}{\A_i} \models q$.
 \item Query $q$ is a \emph{safe} consequence of $\K_\M$, denoted by $\K_\M \models_{\cap} q$, iff $\KB{\T}{\bigcap_{\A_i\in \M}\A_i} \models q$.
 \item Query $q$ is a \emph{majority-based} consequence of $\K_\M$, denoted $\K_\M \models_{maj} q$, iff $\left. \frac{|\A_i: \A_i \in \M,\KB{\T}{\A_i}\models q|}{|\M|}  \right. > 1/2$.
\item  Query $q$ is an \emph{existential} consequence of $\K_\M$, denoted by $\K_\M \models_{\exists} q$ iff $\exists \A_i \in \M$, $\KB{\T}{\A_i} \models q$.
\end{itemize}

Given two inference strategies $s_i$ and $s_j$, we say that  $s_i$ is \emph{more cautious} than $s_j$, denoted $s_i \leq s_j$, when for any consistent MBox $\K_\M$ and any query $q$,  if $\K_\M$$\models_{s_i}$$q$ then $\K_\M$$\models_{s_j}$$q$. The considered inference strategies are totally ordered by $\leq$ as follows:
\begin{equation}
\label{eq:inferrelation}
\cap \leq \forall \leq maj \leq \exists
\end{equation}
\begin{figure}[H]
\centering\scriptsize
\begin{tikzpicture}[ >=stealth']
\node[draw] (safe) {safe inference};
\node[draw, below=.4cm of safe] (forall) {universal inference};
\node[draw, below=.4cm of forall] (majority) {majority-based inference};
\node[draw, below=.4cm of majority] (exist) {existential inference};
\draw [->] (safe) --  (forall);
\draw [->] (forall)  -- (majority);
\draw [->]  (majority) -- (exist);
\end{tikzpicture}
\caption{\small Comparison between inference strategies, where $X$$\longrightarrow$$Y$ means that $X \leq Y$}
\label{sch:comp-inf-stg}
\end{figure}
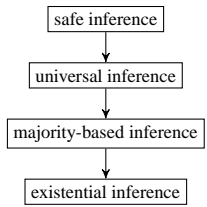

\begin{example}
\label{exp:inferstg}
Let us consider the MBox $\M_1$=$\circ_1(\M)$ given in Example  \ref{exp:mboxmodif}.  We have $\bigcap_{\A_i\in \M}$$\A_i$=$\{A(b)\}$, hence $\K_{\M_{1}}$$\models_{\cap}$$D(b)$. By universal inference, we also have $\K_{\M_{1}}$$\models_{\forall}$$D(a)$. The majority-based inference adds $E(a)$ as a valid conclusion. Indeed, $\KB{\T}{\{B(a),A(b)\}}$$\models$$E(a)$ and $\KB{\T}{\{C(a),A(b)\}}$$\models$$E(a)$ and $|\M_1|$=$3$, hence $\K_{\M_{1}}$$\models_{maj}$$E(a)$. Finally, the existential inference adds $A(a)$ as a valid conclusion.
\end{example}

%%%%%%%%%%%%%%%%%%%%%%%%%%%%%%%%%%%%%%%%%%%%%%%%%%%%%%%%%%%%%%%%%%%%%%%%%%%%%%%%%%%%%%%%%%%%%%%%%%%%%%%%%%%%%%%%%%%%%%%%%%%%%%%%%%%%%%%%%%%%%%%%%%%%%%

\subsection{Inconsistency-Tolerant Semantics = Composite Modifier + Inference Strategy}

We can now define an inconsistency-tolerant query answering semantics by a composite modifier and an inference strategy.

\begin{definition}
\label{def:generic-inf}
Let $\K$=$\KB{\T}{\A}$ be a standard {\kb},
$\circ_i$ be a composite modifier and $s_j$
be an inference strategy. A query $q$ is
said to be an
\emph{$\infr{\circ_i}{s_j}$-consequence} of
$\K$, which is denoted by $\K \infer
{i}{s_j}q$, if it is entailed from the MBox KB
$\KB{\T}{\circ_i(\{\A\})}$ with the inference strategy $s_j$.
\end{definition}

This definition covers the main semantics recalled in Section \ref{sec:dllite}:  AR, IAR, CAR, ICAR and ICR semantics respectively correspond to $\langle \m_1,\forall \rangle$, $\langle \m_1,\cap \rangle$,  $\langle \m_7, \forall \rangle$ $\langle \m_7, \cap \rangle$ and $\langle \m_5, \cap \rangle$.

%----------------------------------------------------------------------------------------------------------------------------------%
%																		       %
%								Section									       %
%																		       %
%----------------------------------------------------------------------------------------------------------------------------------%

\section{Comparison of Inconsistency-Tolerant Semantics w.r.t. Productivity}
\label{sec:inferprod}
We now compare the obtained semantics with respect to productivity, which we formalize as follows.

\begin{definition}
Given two semantics $\infr{\circ_i}{s_k}$ and $\infr{\circ_j}{s_l}$, we say that $\infr{\circ_j}{s_l}$ is \emph{more productive} than $\infr{\circ_i}{s_k}$, and note
$\infr{\circ_i}{s_k} \sqsubseteq \infr{\circ_j}{s_l}$ if, for any KB  $\K$=$\KB{\T}{\A}$ and any query $q$,  if $\K \infer i {s_k} q$ then $\K \infer j {s_l} q$.
\end{definition}

We first pairwise compare semantics defined with the same inference strategy. For each inference strategy, we give necessary and sufficient conditions for
the comparability of the associated semantics w.r.t. productivity. These conditions rely on the inclusion relations  between modifiers (see Figure  \ref{sch:mbox-modif-relats}).

\begin{proposition}\label{prop:safeprod}[Productivity of $\cap$-semantics] See Figure \ref{sch:safe-inf}. It holds that $\langle \circ_i , \cap \rangle \sqsubseteq \langle \circ_j , \cap \rangle$ iff $\circ_j \subseteq \circ_i$ or $\circ_i \subseteq_R \circ_j$ in a bijective way (see Fact 2).
\end{proposition}

\begin{figure}[htb]
\scriptsize\centering
\begin{tikzpicture}[>=stealth']
\tikzstyle{infer}=[rectangle,draw,fill=none];
\node[infer] (freeM1) {$\infr{\ma}{\cap}$};
\node[infer] at (-1,-1) (freeM2) {$\infr{\mb}{\cap}$};
\node[infer] at(1,-1) (freeM5) {$\infr{\me}{\cap}$};
\node[infer, below=.4cm of freeM2] (freeM3) {$\infr{\mc}{\cap}$};
\node[infer, below=.4cm of freeM3] (freeM4) {$\infr{\md}{\cap}$};
\node[infer, below  =.4cm of freeM5] (freeM6) {$\infr{\mf}{\cap}$};
\node[infer, right=.4cm of freeM6] (freeM7) {$\infr{\mg}{\cap}$};
\node[infer, below=.4cm of freeM7] (freeM8) {$\infr{\mh}{\cap}$};

\draw [->] (freeM1) -- (freeM2);
\draw [->] (freeM1) -- (freeM5);
\draw [->] (freeM2) -- (freeM3);
\draw [->] (freeM3) -- (freeM4);
\draw [<-] (freeM3) -- (freeM5);
\draw [->] (freeM5) -- (freeM6);
\draw [->] (freeM5) -- (freeM7);
\draw [->] (freeM7) -- (freeM8);

\node[draw=none] (less) at (4.5,0) {less  productive};
\node[draw=none] (more) at (4.5,-2.8) {more productive};
\draw [->] (less) -- (more);
\end{tikzpicture}
\caption{Relationships between $\cap$-based semantics}
\label{sch:safe-inf}
\vspace{-0.2cm}
\end{figure}
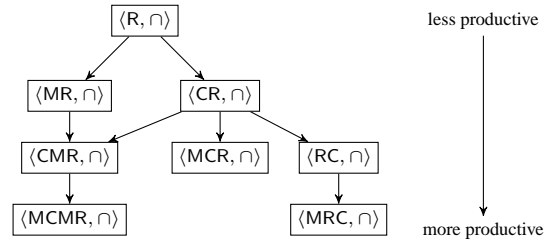

\begin{proposition}\label{prop:univprod}[Productivity of $\forall$-semantics] See Figure \ref{sch:all-inf}. It holds that $\langle \circ_i , \forall \rangle \sqsubseteq \langle \circ_j , \forall \rangle$ iff $\circ_j \subseteq \circ_i$, or $\circ_i \subseteq_R \circ_j$ in a bijective way (see Fact 2) or $\circ_j \subseteq_{cl} \circ_i$.
\end{proposition}

\begin{figure}[htb]
\scriptsize\centering
\begin{tikzpicture}[>=stealth']
\tikzstyle{infer}=[rectangle,draw,fill=none];
\node[infer] (allM1) {$\infr{\ma}{\forall} \equiv \infr{\me}{\forall}$};
\node[infer,  below =.3cm of allM1] (allM6) {$\infr{\mf}{\forall}$};
\node[infer,  left =.3cm of allM6] (allM2) {$\infr{\mb}{\forall}$$\equiv$$\infr{\mc}{\forall}$};
\node[infer,  below = .3cm of allM2] (allM4) {$\infr{\md}{\forall}$};
\node[infer, right =.3cm of allM6] (allM7) {$\infr{\mg}{\forall}$};
\node[infer,  below=.3cm of allM7] (allM8) {$\infr{\mh}{\forall}$};
\draw [->] (allM1) -- (allM2);
\draw [->] (allM2) -- (allM4);
\draw [->] (allM1) -- (allM6);
\draw [->] (allM1) -- (allM7);
\draw [->] (allM7) -- (allM8);

\node[draw=none] (less) at (3.3,0) {less  productive};
\node[draw=none] (more) at (3.3,-1.6) {more productive};
\draw [->] (less) -- (more);

\end{tikzpicture}
\caption{Relationships between $\forall$-based semantics}
\label{sch:all-inf}
\vspace{-0.3cm}
\end{figure}
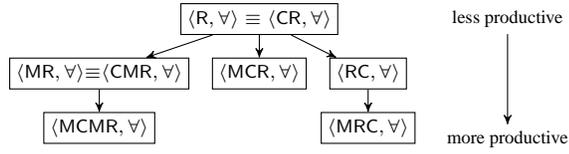

\begin{proposition}\label{prop:majprod} [Productivity of $maj$-semantics] See Figure \ref{sch:maj-inf}. It holds that $\langle \circ_i , maj \rangle \sqsubseteq \langle \circ_j , maj \rangle$ iff $\circ_i \subseteq_R \circ_j$ in a bijective way (see Fact 2) or $\circ_j \subseteq_{cl} \circ_i$ .
\end{proposition}

\begin{figure}[htb]
\scriptsize\centering
\begin{tikzpicture}[>=stealth']
\tikzstyle{infer}=[rectangle,draw,fill=none];
\node[infer] (majM1) {$\infr{\ma}{maj}\equiv \infr{\me}{maj}$};
\node[infer, left  =.4cm of majM1] (majM2) {$\infr{\mb}{maj}\equiv \infr{\mc}{maj}$};
\node[infer,  below=.4cm of majM2] (majM4) {$\infr{\md}{maj}$};
\node[infer,  below=.4cm of majM1] (majM7) {$\infr{\mg}{maj}$};
\node[infer,  below=.4cm of majM7] (majM8) {$\infr{\mh}{maj}$};
\node[infer,  left=.1cm of majM7] (majM6) {$\infr{\mf}{maj}$};

\draw [->] (majM1) -- (majM7);
\node[draw=none] (less) at (2.3,0) {less  productive};
\node[draw=none] (more) at (2.3,-1.7) {more productive};
\draw [->] (less) -- (more);
\end{tikzpicture}
\caption{Relationships between $maj$-based semantics}
\label{sch:maj-inf}
\vspace{-0.3cm}
\end{figure}
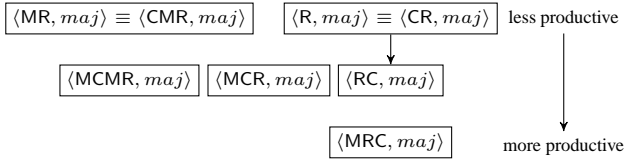

\begin{proposition}\label{prop:existprod}[Productivity of $\exists$-semantics] See Figure \ref{sch:exist-inf}. It holds that $\langle \circ_i , \exists \rangle \sqsubseteq \langle \circ_j , \exists  \rangle$ iff $\circ_i \subseteq_R \circ_j$ (in particular $\circ_i \subseteq \circ_j$ or  $\circ_i \subseteq_{cl} \circ_j$) or $\circ_j \subseteq_{cl} \circ_i$ .
\end{proposition}

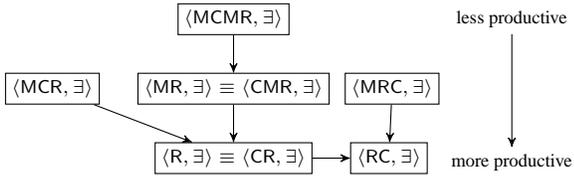
\begin{figure}[H]
\scriptsize\centering
\begin{tikzpicture}[ >=stealth']
\tikzstyle{infer}=[rectangle,draw,fill=none];
\node[infer] (exM4) {$\infr{\md}{\exists}$};
\node[infer, below = .5cm of exM4] (exM2) {$\infr{\mb}{\exists}\equiv\infr{\mc}{\exists}$};
\node[infer,  below = .5cm of exM2] (exM1) {$\infr{\ma}{\exists}  \equiv \infr{\me}{\exists}$};
\node[infer, left = .5 cm of exM2] (exM6) {$\infr{\mf}{\exists}$};
\node[infer, right = .5cm of exM1] (exM7) {$\infr{\mg}{\exists}$};
\node[infer, right=  .2  cm of exM2] (exM8) {$\infr{\mh}{\exists}$};
\draw [<-] (exM1) -- (exM2);
\draw [<-]  (exM2) -- (exM4);
\draw [<-]  (exM1) -- (exM6);
\draw [<-]  (exM7) -- (exM8);
\draw [->]  (exM1) -- (exM7);

\node[draw=none] (less) at (3.7,0) {less  productive};
\node[draw=none] (more) at (3.7,-1.9) {more productive};
\draw [->] (less) -- (more);
\end{tikzpicture}
\caption{Relationships between $\exists$-based semantics}
\label{sch:exist-inf}
\vspace{-0.3cm}
\end{figure}

%%%%%%%%%%%%%%%%%%%%%%%%%%%%%%%%%%%%%%%%%%%%%%%%%%%%%%%%%%%%%%%%%%%%%%%%%%%%%%%%%%%%%%%%%%%%%%%%%%%%%%%%%%%%%%%%%%%%%%%%%%%%%%%%%%%%%%%%%%%%%%%%%%%%%%%%%

We now extend previous results to any pair of semantics, possibly based on different inference strategies.

\begin{theorem}\label{th:prodsem}[Productivity of semantics] The inclusion relation $\sqsubseteq$ is the smallest relation that contains
the inclusions $\langle \circ_i , s_k \rangle \sqsubseteq \langle \circ_j , s_k\rangle$  defined by Propositions   \ref{prop:safeprod}-\ref{prop:existprod} and satisfying the two following conditions:
\begin{enumerate}
\item for all $s_j$, $s_p$ and $o_i$, if  $s_j \leq s_p$ then $\langle \circ_i , s_j \rangle \sqsubseteq \langle \circ_i, s_p \rangle$.
\item it is transitive.
\end{enumerate}
\end{theorem}

Theorem \ref{th:prodsem} is an important result. It states that the productivity relation can only be obtained from Figures \ref{sch:safe-inf}-\ref{sch:exist-inf} (\resp{} Propositions \ref{prop:safeprod}-\ref{prop:existprod}) and some composition of the relations. No more inclusion relations hold. In particular when $s_i > s_j$, it holds that  $\forall k, \forall l$, $\infr{\circ_k}{s_i} \not \sqsubseteq \infr{\circ_l}{s_j}$, which means that there exist a query $q$ and a KB $\mathcal K$ such that $q$ is an $\infr{\circ_k}{s_i}$-consequence of $\mathcal K$ but not  an $\infr{\circ_l}{s_j}$-consequence of $\mathcal K$. Note that this holds already for DL-Lite$_R$ KBs. 

\begin{proof1}[Sketch]
Condition 1 holds by definition of $\leq$. Transitivity holds by definition of $\sqsubseteq$. To show that there are no other inclusions, we prove two lemmas:  
for all $\infr{\circ_i}{s_j}$ and $\infr{\circ_k}{s_p}$,  (1) if $s_p < s_j$ then  $\infr{\circ_i}{s_j} \not \sqsubseteq \infr{\circ_k}{s_p}$; and (2) if $\infr{\circ_i}{s_j} \sqsubseteq \infr{\circ_k}{s_p}$ and $s_j < s_p$, then $\infr{\circ_i}{s_p} \sqsubseteq \infr{\circ_k}{s_p}$.
\end{proof1}

Lastly, it is important to note that when the initial \kb{} is consistent, all semantics collapse with standard entailment, namely:

\begin{proposition}
\label{prop:cons-case}
Let $\K$ be a consistent standard \kb{}. Then: $\forall$$s$$\in$$\{\cap$,$\forall$,$maj$, $\exists\}$, $\forall i: 1,..., 8, $ $\K$$\infer i s$$q$ iff  $\K$ $\models$ $q$.
\end{proposition}

%----------------------------------------------------------------------------------------------------------------------------------%
%																		       %
%								Section									       %
%																		       %
%----------------------------------------------------------------------------------------------------------------------------------%

\section{Conclusion}

This paper provides a general and unifying framework for inconsistency-tolerant query answering. On the one hand, our logical setting based on existential rules includes previously considered languages. On the other hand, viewing an inconsistency-tolerant semantics as a pair composed of a modifier and an inference strategy allows us to include the main known semantics and to consider new ones.  We believe that the choice of semantics depends on the applicative context, namely the features of the semantics, i.e rationality properties, complexity (which we have studied, but not presented in this paper) and productivity with respect to the applicative context. 
In particular, cardinality-based selection allows us to counter troublesome assertions that conflict with many others. In some contexts, requiring to find an answer in all selected repairs can be too restrictive, hence the interest of majority-based semantics, which are more productive than universal semantics, without being as productive as the adventurous existential semantics.
As for future work, we plan consider other inference strategies such as the argued inference, parametrized inferences, etc. We also want to adapt the framework to  belief change problems, like merging or revision.

\section*{Acknowledgment}
This work has been supported by the French National Research Agency. ASPIQ project ANR-12-BS02-0003.

\bibliographystyle{aaai}
\small\bibliography{kr16-short}
\onecolumn
%%%%%%%%%%%%%%%%%%%%%%%%%%%%%%%%%%%%%%%%%%%%%%%%%%%%%%%%%%
\section{Appendix}

\vspace{1cm}

In this appendix, we provide details on the proofs. 

\subsection*{Section 3: A Unified Framework for Inconsistency-Tolerant Query Answering}

\medskip
\noindent \textbf {Theorem 1}. 
\emph{Let $\K_\M$=$\KB{\T}{\M=\{\A\}}$ be a possibly inconsistent  \kb{}. Then for any composite modifier $\circ_c$ that can be obtained by a finite combination of the elementary modifiers $\inclf$, $\lexf$, $\clf$, there exists a composite modifier $\circ_i$ in $\{\circ_1 \ldots \circ_8\}$ (see Table \ref{tab:abox-modifs}) such that $\circ_c(\M)$=$\circ_i(\M)$.}

\medskip
\noindent{\textbf{ Lemma \ref{lem:equi-modif}} 
\emph{For any MBox $\M$, the following holds:
\begin{enumerate}
\item $\cl{\cl{\M}}$=$\cl{\M}$, $\incl{\incl{\M}}$=$\incl{\M}$ and  $\lex{\lex{\M}}$ =$\lex{\M}$.
\item Let $\circ_{d}$ be any composite modifier. Then: \\(a) $\cl{\circ_{d}(\cl{\M})}=\circ_{d}(\cl{\M})$, and \\
(b) $\incl{\circ_{d}(\incl{\M})}=\circ_{d}(\incl{\M})$.
\end{enumerate}
}
\begin{proof}%[Proof of Lemma \ref{lem:equi-modif}]
The proof of the idempotence of $\inclf$ follows from the facts that: i) $\forall \A_i \in \incl{\M}, \KB{\T}{\A_i}$ is consistent and ii) if $\KB{\T}{\A_i}$ is consistent, then $\incl{\A_i}=\{\A_i\}$.  The proof of the the idempotence of $\lexf$ follows from the facts that: i) $\forall \A_i\in \lex{\M}, \forall \A_j \in \lex{\M}$, we have $|\A_i|=|\A_j|$ ii) if $\forall \A_i\in \lex{\M}, \forall \A_j \in \lex{\M}, |\A_i|=|\A_j|$ then $\lex{\M}=\M$. 
For the idempotence of $\clf$, it is enough to show that for a given $\A\in\M, \cl{\cl{\A}}=\cl{\A}$. From the definition of $\clf$, clearly we have $\cl{\A}\subseteq  \cl{\cl{\A}}$. Now assume that $f\in\cl{\cl{\A}}$ but  $f\notin\cl{\A}$. Let $B_f \subseteq \cl{\A}$ be the subset that allows to derive $f$, namely $\KB{\T_p}{B_f}\models f$. Now for each element $x$ of $B_f$, we have $\KB{\T_p}{\A}\models x$. Then clearly, $\KB{\T_p}{\A}\models f$. 

Regarding item (2.a), if $\circ_{d}$ is an elementary modifier then it can be either $\clf$, $\lexf$, or $\inclf$. If $\circ_{d}=\clf$ then the result holds since $\clf$ is idempotent.  If $\circ_{d}=\lexf$ then the selected elements from $\lex{\cl{\M}}$ are closed sets of assertions  since  $\lexf$ only discards some elements of $\cl{\M}$ but does not change the content of remaining elements. Lastly, let us consider the case where $\circ_{d}=\inclf$. Again $\forall \A' \in \incl{\cl{\M}}, \A'=\cl{\A'}$. Let us recall that $\A'$ is a maximally consistent subset  of $\A\in\cl{\M}$, with $\A=\cl{\A}$. If $\A'\neq\cl{\A}$ this means that $\exists f \in \cl{\A'}$ (hence $f\in\A$) such that $f\notin\A'$ despite the fact that $\KB{\T}{\A'}\models f$. This is impossible since $\A'$ should be a maximal consistent subbase of $\A$. Since each $\circ_{d}\in\{\clf,\lexf,\inclf\}$ applied on closed a ABox  preserves the closeness property, then clearly a composite modifier also preserves this closeness property.

The proof of item (2.b) follows immediately from the fact that i) $\forall \A_i\in\incl{\M}$, $\KB{\T}{\A_i}$ is consistent, ii) if $\M$ is consistent, then $\forall \circ_{d}\in\{\clf,\lexf,\inclf\}$ yields a consistent  subbase, and iii) $\incl{\M}=\M$ if $\M$ is consistent. 
\end{proof}

\begin{proof} [Proof of Theorem \ref{th:eight-modifier}]
The proof relies on Lemma \ref{lem:equi-modif} (see also the explanations following Lemma \ref{lem:equi-modif} in the paper). 
\end{proof}

%%%%%%%%%%%%%%%%%%%%%%%%%%%%%%%%%%%%%%%%%%%%%
\noindent
\textbf{Justification of Figure 2 (Inclusion relations between composite modifiers)}: see following Proposition \ref{prop:cplex-modif-rels}, Example \ref{ex-modifiers1}, Proposition \ref{prop:cplex-modif-rels2}  and Example \ref{ex-modifiers2}. 

\begin{proposition}[Part of the proof of Figure \ref{sch:mbox-modif-relats}]
\label{prop:cplex-modif-rels}
Let $\K_\M$=$\KB{\T}{\M=\{\A\}}$ be an inconsistent  \kb{}. Let $\{\M_1$,...,$\M_8\}$ be the MBoxes obtained by the eight composite modifiers $\{\circ_1,...,\circ_8\}$ summarized in Table \ref{tab:abox-modifs}. Then: 
\begin{enumerate}
\item $\M_2 \subseteq \M_1$. 
\item $\M_4 \subseteq \M_3$. 
\item $\M_6 \subseteq \M_5$. 
\item $\M_8 \subseteq \M_7$.
\item $\M_3=\cl{\M_2}$.
\item $\M_5=\cl{\M_1}$.
\item $\M_3 \subseteq \M_5$. 
\item $\M_5 \subseteq_{R} \M_7$.
\end{enumerate}
\end{proposition}

\begin{proof}%[Proof of Proposition \ref{prop:cplex-modif-rels}]
%The proofs are as follows: 
$~$

\begin{itemize}
\item Items 1-4 follow from the definition of the elementary modifier $\lexf$. Since $\lexf$ selects subsets of $\M$ having maximal cardinality. Namely, given $\M$ an MBox, we have $\lex{\M}\subseteq \M$. Hence relations  $\M_4\subseteq\M_3$, $\M_2\subseteq\M_1$, $\M_6\subseteq\M_5$, and $\M_8\subseteq\M_7$ holds.

\item Items 5-6 follow immediately from the definition of the elementary modifier $\clf$, hence we trivially have $\M_5=\cl{\M_1}$ and $\M_3=\cl{\M_2}$.

\item Let us show that $\M_2 \subseteq_{cl} \M_5$, namely $\forall A$$\in$$\M_2$,$\exists B\in \M_5$ such that $B=\Cl(\A)$. The proof is immediate. Recall that $\M_2\subseteq\M_1$, hence $\forall A\in\M_2$ we also have $A\in\M_1$. Recall also that $\M_5=\cl{\M_1}$. This means  that $\forall A\in \M_2, \exists B\in \M_5$ such that $B=\Cl(A)$.

\item Regarding the proof of Item 8,  we have $\M_2  \subseteq_{cl} \M_5$. This means that $\forall A \in \M_2$, there exists $B \in \M_5$ such that $B=\Cl(A)$. Said differently, $\forall A\in \M_2$, we have $\Cl(A)\in\M_5$. Since  $\M_3$=$\cl{\M_2}$, we conclude that $\M_3\subseteq\M_5$.  

\item  We now show that $\M_5$$\subseteq_{R}$$\M_7$.  Let $B$$\in$$\incl{\{\A\}}$ and let us show that there exists a set of assertions $X$ such that $\cl{\{B\}}$$\subseteq$$X$ and $X$$\in$$\M_7$. Since $B$$\in$$\incl{\{\A\}}$, this means by definition that $B\subseteq \A$ and hence $B \subseteq \cl{\A}$. Now, $B$ is consistent, this means that there exists $R\in \incl{\cl{\A}}=\M_7$ such that $B \subseteq R$. From Lemma  \ref{lem:equi-modif}, $R$ is a closed set of assertions, then this means that $\Cl(B)\subseteq R$.
\end{itemize}
\end{proof}

\begin{example}[Counter-examples showing that there are no reciprocal edges in Figure \ref{sch:mbox-modif-relats}]
\label{ex-modifiers1}
$~$

\begin{enumerate}

\item The converse of $\M_2 \subseteq \M_1$ does not hold.\\ 
Let $\T$=$\{B$$\isa$$C$,$C$$\isn$$D\}$ and $\M$=$\{\{B(a)$, $C(a)$, $D(a)\}\}$. \\ It is easy to check that $\KB{\T}{\M}$ is inconsistent. We have: \\
$\M_1$=$\incl{\M}$=$\{\{C(a), B(a)\},\{D(a)\}\}$, and\\
$\M_2=\lex{\M_1}=\{\{C(a),B(a)\}\}$. \\
One can check that $\M_1 \nsubseteq_{R} \M_2$.\\

\item The converse of $\M_4 \subseteq \M_3$ does not hold.\\ 
Let $\T$=$\{A\isa B$, $B\isn C\}$ and $\M$=$\{\{A(a),C(a)\}\}$.  \\ It is easy to check that $\KB{\T}{\M}$ is inconsistent.  We have: \\
$\M_1$=$\incl{\M}$=$\{\{A(a)\}, \{C(a)\}\}$, \\
$\M_2=\lex{\M_1}=\{\{A(a)\},\{C(a)\}\}$,  \\
$\M_3=\cl{\M_2}=\{\{A(a), B(a)\}, \{C(a)\}\}$, and \\
$\M_4=\lex{\M_3}=\{\{A(a), B(a)\}\}$.  \\
One can check that $\M_3 \nsubseteq_{R} \M_4$.\\

\item  The converse of $\M_6\subseteq \M_5$ does not hold.\\
Let $\T$=$\{B\isa C$, $C\isn D\}$ and $\M$=$\{\{B(a),D(a)\}\}$. \\ It is easy to check that $\KB{\T}{\M}$ is inconsistent.  We have: \\
$\M_1$=$\incl{\M}$=$\{\{B(a)\},\{D(a)\}\}$, \\
$\M_5=\cl{\M_1}$=$\{\{C(a),B(a)\},\{D(a)\}\}$, and \\
$\M_6=\lex{\M_5}$=$\{\{C(a),B(a)\}\}$\\
One can check that $\M_5\nsubseteq_{R} \M_6$. \\

\item  The converse of $\M_8\subseteq \M_7$ does not hold.\\ 
Let $\T$=$\{A\isa B$, $B\isn D\}$ and $\M$=$\{\{A(a),D(a)\}\}$.  \\ It is easy to check that $\KB{\T}{\M}$ is inconsistent.  We have: \\
$\cl{\M}$=$\{\{A(a),B(a),D(a)\}\}$, \\
$\M_7=\incl{\cl{\M}}$=$\{\{A(a),B(a)\},\{D(a)\}\}$, and \\
$\M_8=\lex{\M_7}$=$\{\{A(a),B(a)\}\}$\\
One can check that $\M_7\nsubseteq_{R} \M_8$. \\

\item The converse of $\M_3 \subseteq \M_5$ does not hold.\\ 
Let $\T$=$\{A\isa B$, $B\isa C$, $C\isn D\}$ and $\M$=$\{\{A(a), B(a), D(a)\}\}$.  \\ It is easy to check that $\KB{\T}{\M}$ is inconsistent.  We have: \\
$\M_1$=$\incl{\M}$=$\{\{A(a), B(a)\}$, $\{D(a)\}\}$,\\
$\M_5=\cl{\M_1}=\{\{A(a),B(a),C(a)\}$, $\{D(a)\}\}$, \\
$\M_2=\lex{\M_1}$=$\{\{A(a),B(a)\}\}$, and \\
$\M_3=\cl{\M_2}$=$\{\{A(a),B(a),C(a)\}\}$. \\
One can check that $\M_5 \nsubseteq_{R} \M_3$.\\

\item The converse of $\M_5 \subseteq_{R} \M_7$ does not hold.\\ 
Let $\T$=$\{A\isn B$, $B\isa D\}$ and $\M$=$\{\{A(a), B(a)\}\}$.  \\ It is easy to check that $\KB{\T}{\M}$ is inconsistent.  We have: \\
$\cl{\M}$=$\{\{A(a),B(a),D(a)\}\}$,\\
$\M_7=\incl{\cl{\M}}$=$\{\{A(a),D(a)\},\{B(a),D(a)\}\}$,\\
$\M_1$=$\incl{\M}$=$\{\{A(a)\},\{B(a)\}\}$, and \\
$\M_5=\cl{\M_1}$=$\{\{A(a)\},\{B(a),D(a)\}\}$, \\
One can check that $\M_7 \nsubseteq_{R} \M_5$.\\
\end{enumerate}
\end{example}

\begin{corollary}
Let $\K_\M$=$\KB{\T}{\M=\{\A\}}$ be an inconsistent  \kb{}. Let $\{\M_1$,...,$\M_8\}$ be the MBoxes obtained by the eight composite modifiers $\{\circ_1$,...,$\circ_8\}$ summarized in Table \ref{tab:abox-modifs}. Then:
\begin{enumerate}
%%%M1
\item $\forall \A_i \in \M_3$, $\exists \A_j\in \M_1$ such that $\A_i=\Cl(\A_j)$.
\item $\forall \A_i \in \M_4$, $\exists \A_j\in \M_1$ such that $\A_i=\Cl(\A_j)$.
\item $\forall \A_i \in \M_6$, $\exists \A_j\in \M_1$ such that $\A_i=\Cl(\A_j)$.
\item $\forall \A_i \in \M_1$, $\exists \A_j\in \M_7$ such that $\A_i\subseteq\A_j$.
\item $\forall \A_i \in \M_1$, $\exists \A_j\in \M_8$ such that $\A_i\subseteq\A_j$.
%%%M2--1,3,4,5,7-----6,8
\item $\forall \A_i \in \M_4$, $\exists \A_j\in \M_2$ such that $\A_i=\Cl(\A_j)$.
\item $\forall \A_i \in \M_2$, $\exists \A_j\in \M_7$ such that $\A_i\subseteq\A_j$.
%%%M3, 1,2,4,5,7----6,8
\item $\forall \A_i \in \M_3$, $\exists \A_j\in \M_7$ such that $\A_i\subseteq\A_j$.
%%%M4, 1,2,4,5,7----6,8
\item $\forall \A_i \in \M_4$, $\exists \A_j\in \M_7$ such that $\A_i\subseteq\A_j$.
%%%M5, 1,2,4,5,7----8
\item $\forall \A_i \in \M_5$, $\exists \A_j\in \M_8$ such that $\A_i\subseteq\A_j$.
\end{enumerate}
\end{corollary}

\begin{proposition}[Part of the proof of Figure \ref{sch:mbox-modif-relats}]
\label{prop:cplex-modif-rels2}
%Let $\K_\M$=$\KB{\T}{\M=\{\A\}}$ be an inconsistent  \kb{}. 
Let $\{\circ_1$,...,$\circ_8\}$ be the eight composite modifiers summarized in Table \ref{tab:abox-modifs}.Then: 
\begin{enumerate}
\item There exists $\M$ such that $\circ_6(\M)$ and  $\circ_8(\M)$ are incomparable. 
\item There exists $\M$ such that $\circ_2(\M)$ and $\circ_6(\M)$ are incomparable.
\item There exists $\M$ such that $\circ_3(\M)$ and $\circ_6(\M)$ are incomparable.
\item There exists $\M$ such that $\circ_4(\M)$ and $\circ_6(\M)$ are incomparable.
\item There exists $\M$ such that $\circ_2(\M)$ and $\circ_8(\M)$ are incomparable.
\item There exists $\M$ such that $\circ_3(\M)$ and $\circ_8(\M)$ are incomparable.
\item There exists $\M$ such that $\circ_4(\M)$ and $\circ_8(\M)$ are incomparable.
\end{enumerate}
\end{proposition}

\begin{example}[Examples that prove Proposition \ref{prop:cplex-modif-rels2}] 
\label{ex-modifiers2}

The following examples prove the statements in Proposition \ref{prop:cplex-modif-rels2}. 

\begin{enumerate}
\item There exists $\M$ such that  $\circ_6(\M)$ and $\circ_8(\M)$ are incomparable.\\ 
Let $\T$=$\{B\isn C$, $B\isa A$, $C\isa A, A\isn D, D\isa E, E\isa F\}$ and $\M$=$\{\{A(a), B(a), C(a),D(a)\}\}$.  \\ It is easy to check that $\KB{\T}{\M}$ is inconsistent.  We have: \\
$\M_1=\incl{\M}=\{\{A(a),C(a)\}$, $\{A(a),B(a)\}$, $\{D(a)\}\}$, and \\
$\M_5=\cl{\M_1}$=$\{\{A(a),C(a)\}$, $\{A(a),B(a)\}$, $\{D(a),E(a),F(a)\}\}$, and \\
$\M_6=\lex{\M_5}$=$\{\{D(a),E(a),F(a)\}\}$,\\
$\cl{\M}$=$\{\{A(a),B(a),C(a),D(a),E(a),F(a)\}\}$, \\
$\M_7=\incl{\cl{\M}}$=$\{\{A(a),C(a),E(a),F(a)\}$, $\{A(a),B(a),E(a),F(a)\}$, $\{D(a),E(a),F(a)\}\}$, and \\
$\M_8=\incl{\M_7}=\{\{A(a),C(a),E(a),F(a)\}$, $\{A(a),B(a),E(a),F(a)\}\}$ \\
One can check that $\M_6$ and $\M_8$ are incomparable.\\

\item There exists $\M$ such that $\circ_2(\M)$, $\circ_3(\M)$ and $\circ_4(\M)$ are incomparable with $\circ_6(\M)$.\\ 
Let $\T$=$\{A\isn B$, $C\isa A$, $B\isa D$, $D\isa F\}$ and $\M$=$\{\{A(a),C(a),B(a)\}\}$.  \\ It is easy to check that $\KB{\T}{\M}$ is inconsistent.  We have: \\
$\M_1=\incl{\M}$=$\{\{A(a),C(a)\},\{B(a)\}\}$, \\
$\M_2=\lex{\M_2}$=$\{\{A(a),C(a)\}\}$, \\
$\M_5=\cl{\M_1}$=$\{\{A(a),C(a)\},\{B(a),D(a),F(a)\}\}$, \\
$\M_6=\lex{\M_5}$=$\{\{B(a),D(a),F(a)\}\}$, \\
One can check that $\M_2$ is incomparable with  $\M_6$.\\
We have also $\M_2$=$\M_3$=$\M_4$=$\{\{A(a),C(a)\}\}$, So, we conclude that $\M_3$ and $\M_4$ are incomparable with $\M_6$.

\item There exists $\M$ such that $\circ_2(\M)$, $\circ_3(\M)$ and $\circ_4(\M)$ are incomparable with $\circ_8(\M)$.\\ 
Let $\T$=$\{B\isa A$, $C\isa A, A\isn D, E\isa D, D\isa F\}$ and $\M$=$\{\{A(a),D(a),E(a)\}\}$.  \\ It is easy to check that $\KB{\T}{\M}$ is inconsistent.  We have: \\
$\M_1=\incl{\M}=\{\{A(a)\}$, $\{D(a),E(a)\}\}$, and \\
$\M_2=\lex{\M_1}$=$\{D(a),E(a)\}\}$, and \\
$\M_3=\M_4$=$\{\{D(a),E(a),F(a)\}\}$,\\
$\cl{\M}$=$\{\{A(a),B(a),C(a),D(a),E(a),F(a)\}\}$, \\
$\M_7=\incl{\cl{\M}}$=$\{\{A(a),C(a),B(a),F(a)\}$, $\{D(a),E(a),F(a)\}\}$, and \\
$\M_8=\incl{\M_7}=\{\{A(a),C(a),B(a),F(a)\}$, \\
One can check that $\M_2$,  $\M_3$ and $\M_4$ are incomparable with $\M_8$.\\
\end{enumerate}
\end{example}

%%%%%%%%%%%%%%%%%%%%%%%%%%%%%%%%%%%%%%%%%%%%%%%%%%%%%%%%%%%%%%%%%%%%%%%%%%%%%%%%%%%%%%%%%%%%%%%%%%%%%%%%%%%%%%%%%%%%%%%%%%%%%%%%%%%%%%%%%%%%%%%%%%%%%%

%\subsection*{Inference Strategies for Querying an MBox}

\begin{proposition}[Proof of Equation \ref{eq:inferrelation} and Figure \ref{sch:comp-inf-stg}]
\label{prop:inf-relations}
Let $\M$ be a consistent MBox \wrt{} a TBox $\T$. Let $q$ be a query. Then:     
\begin{enumerate}
\item if $\KB{\T}{\M}\models_{\cap} q$ then $\KB{\T}{\M}\models_{\forall}q$. 
\item if $\KB{\T}{\M}\models_{\forall}q$ then $\KB{\T}{\M}\models_{maj}q$. 
\item if $\KB{\T}{\M}\models_{maj}q$ then $\KB{\T}{\M}\models_{\exists}q$. 
\end{enumerate}
\end{proposition}

\begin{proof}[Proof of Proposition \ref{prop:inf-relations}]
Item 1 holds from the fact that $\forall \A_i \in \M$, we have $(\bigcap_{\A_i\in \M}\A_i) \subseteq \A_i$.  Item 2 holds due to the fact that universal consequence requires that $q$ follows from all ABoxes in $\M$. Hence, $q$ holds in more than the half of $\A_i$'s in $\M$.  Item 3 follows from the fact that a query is considered as valid using majority-based consequence relation if it is confirmed by more than the half of $\A_i \in \M$. Hence $q$ follows from at least one ABox.
\end{proof}

%%%%%%%% MLM: supprim ces lemmes car ne semblent pas servir %%%%%%%%%%%%%%%%%%%%%%%%
Finally, the following two lemmas about the cautiousness relation will be used later. Lemma \ref{lem:relation-inf} considers two MBoxes,  with one included in the other. Lemma \ref {lem:relation-inf2} considers two MBoxes, where one is the positive closure of the other. 

\begin{lemma}
\label{lem:relation-inf}
Let $\M_1$ and $\M_2$ be two consistent MBoxes \wrt{} a TBox $\T$ such that $\M_1$$\subseteq$$\M_2$. Let $q$ be a query. Then:
\begin{enumerate}
\item If $\KB{\T}{\M_2} \models_{\forall} q$ then $\KB{\T}{\M_1} \models_{\forall} q$.
\item If $\KB{\T}{\M_2} \models_{\cap} q$ then $\KB{\T}{\M_1} \models_{\cap} q$.
\item There are $\M_1$ and $\M_2$ such that the majority-based inference yields incomparable results. 
\item If $\KB{\T}{\M_1} \models_{\exists} q$ then $\KB{\T}{\M_2} \models_{\exists} q$.
\end{enumerate} 
\end{lemma}

\begin{proof} %[Proof of Lemma \ref{lem:relation-inf}]
The proof is immediate. For item 1, if $q$ holds in all $\A_i$ of $\M_2$ then trivially it holds in all $\A_j$ of $\M_1$ (since $\M_1 \subseteq \M_2$).  Item 2 holds due to the fact that $\M_1 \subseteq \M_2$ implies that $\bigcap_{\A_i\in\M_2}\A_i \subseteq \bigcap_{\A_j\in\M_1}\A_j$. Lastly, from item 4, if there exists an $\A_i$ in $\M_1$ where $q$ holds, then such $\A_i$  also exists in $\M_2$. 
\end{proof}

\begin{example}[Counter-examples associated with Lemma \ref{lem:relation-inf}]
The converse of Items 1 and 2  and 4 does not hold, as shown by the following counter-example.  
Let $\T=\emptyset$, $\M_1$=$\{B(a)\}$ and  $\M_2$=$\{\{B(a)\}$,$\{B(c)\}$, $\{B(c)\}\}$. First, note that $\M_1 \subseteq \M_2$.
Clearly $\KB{\T}{\M_1}\models_{\forall} B(a)$ (\resp{} $\KB{\T}{\M_1}\models_{\cap} B(a)$) holds, while $\KB{\T}{\M_2}\models_{\forall} B(a)$ (\resp{} $\KB{\T}{\M_2}\models_{\cap} B(a)$) does not hold.  Similarly $\KB{\T}{\M_2}\models_{\exists} B(c)$ holds, while $\KB{\T}{\M_1}\models_{\exists} B(c)$ does not hold. 

Regarding majority-based inference, one can check that $\KB{\T}{\M_1}\models_{maj} B(a)$ holds while $\KB{\T}{\M_2}\models_{maj} B(a)$ does not hold. And $\KB{\T}{\M_2}\models_{maj} B(c)$ holds while $\KB{\T}{\M_1}\models_{maj} B(c)$ does not hold. 
\end{example}

\begin{lemma}
\label{lem:relation-inf2}
Let $\M_1$ and $\M_2$ be two consistent MBoxex \wrt{} $\T$. Let $\M_2$ be the positive closure of $\M_1$. Let $q$ be a Boolean query. Then:
\begin{enumerate}
\item $\KB{\T}{\M_1} \models_{\forall} q$ iff $\KB{\T}{\M_2} \models_{\forall} q$.
\item $\KB{\T}{\M_1} \models_{maj} q$ iff $\KB{\T}{\M_2} \models_{maj} q$.
\item if $\KB{\T}{\M_1} \models_{\cap} q$ then $\KB{\T}{\M_2} \models_{\cap} q$.
\item $\KB{\T}{\M_1} \models_{\exists}q$ iff $\KB{\T}{\M_2} \models_{\exists}q$.
\end{enumerate} 
\end{lemma}

\begin{proof}[Proof of Lemma \ref{lem:relation-inf2}]
The proof is again immediate. Items 1, 2 and  4  follow from the fact that, if $\A$ is a consistent ABox with $\T$, then $\KB{\T}{\A}\models q$ iff $\KB{\T}{\Cl(\A)}\models q$. Item 3 follows from the fact that $\A_i \subseteq \Cl(\A_i)$ for each $\A_i \in \M_1$. Hence $\bigcap_{\A_i\in\M_1}\A_i \subseteq \bigcap_{\A_i \in\M_1}\Cl(\A_i) = \bigcap_{\A_j \in\M_2} \A_j$. 
\end{proof}

%%%%%%%%%%%%%%%%%%%%%%%%%%%%%%%%%%%%%%%%%%%%%%%%%%%%%%%%%%%%%%%%%%%%%%%%%%%%%%%%%%%%%%%%%%%%%%%%%%%%%%%%%%%%%%%%%%%%%%%%%%%%%%%%%%%%%%%%%%%%%%%%%%%%%%

%----------------------------------------------------------------------------------------------------------------------------------%
%																		       %
%								Section									       %
%																		       %
%----------------------------------------------------------------------------------------------------------------------------------%

\subsection*{Section 4: Comparison of Inconsistency-Tolerant Semantics w.r.t. Productivity}

\medskip

\paragraph{Proof of Figure  \ref{sch:safe-inf} (intersection-based semantics)} The relation pictured in the figure is proved by the following propositions and examples. 

\begin{proposition}[Proof of Figure \ref{sch:safe-inf}, Part 1]
Let $\K_\M$=$\KB{\T}{\M=\{\A\}}$ be an inconsistent  \kb{}. Let $\M_1$,...,$\M_8$ be the MBoxes obtained by  applying the eight modifiers $\{\circ_1,...,\circ_8\}$, given in Table \ref{tab:abox-modifs}, on $\M$. Let $q$ be a Boolean query. Then: 
\begin{enumerate}
\item If $q$ is a safe conclusion of $\KB{\T}{\M_1}$ then $q$ is a safe conclusion of $\KB{\T}{\M_2}$.

\item If $q$ is a safe conclusion of $\KB{\T}{\M_1}$ then $q$ is a safe conclusion of $\KB{\T}{\M_5}$. 

\item If $q$ is a safe conclusion of $\KB{\T}{\M_2}$ then $q$ is a safe conclusion of $\KB{\T}{\M_3}$. 

\item If $q$ is a safe conclusion of $\KB{\T}{\M_3}$ then $q$ is a safe conclusion of $\KB{\T}{\M_4}$. 

\item If $q$ is a safe conclusion of $\KB{\T}{\M_5}$ then $q$ is a safe conclusion of $\KB{\T}{\M_3}$. 

\item if $q$ is a safe conclusion of $\KB{\T}{\M_5}$ then $q$  is a safe conclusion of $\KB{\T}{\M_6}$. 

\item If $q$ is  a safe conclusion of $\KB{\T}{\M_5}$ then $q$  is a safe conclusion of $\KB{\T}{\M_7}$.
 
\item If $q$ is  a safe conclusion of $\KB{\T}{\M_7}$ then $q$  is a safe conclusion of $\KB{\T}{\M_8}$. 
\end{enumerate}
\end{proposition}

\begin{proof}
The proof is as follows:
\begin{enumerate}
\item For items 1, we have $\M_2\subseteq \M_1$, then following Item 2 of Lemma \ref{lem:relation-inf},  if $\langle \M_1,\cap \rangle$ implies a query $q$ then $\langle \M_2, \cap \rangle$ implies it also. The proof follow similarly for Items 4,5, 6 and 8 since $\M_4\subseteq \M_3$, $\M_3\subseteq \M_5$, $\M_6\subseteq \M_5$, and $\M_8\subseteq \M_7$.
 
\item For items 2 and 3, we have $\M_5=\cl{\M_1}$ and $\M_3=\cl{\M_2}$. Then following Item 3 of  Lemma \ref{lem:relation-inf2}, if a query holds in $\langle \M, \cap \rangle$ then it also holds in $\langle \cl{\M},\cap \rangle$. 

\item For item 7, we have $\forall A \in\M_5, \exists B \in \M_7$ such that $A\subseteq B$. Let $A(a)\in\bigcap_{\A_i \in \M_5} \A_i$. Then one can check that there is no conflict $\C$ in $\KB{\T}{\Cl(\M)}$ such that $A(a)\in \C$. Indeed, assume that such conflict exists. Then this means that there exists $B(a)\in \Cl(\M)$ where $\KB{\T}{\{(A(a),B(a)\}}$ is conflicting. Two options: 

i) $B(a)\in \M$. This means that there exists a maximally consistent subset $X$ of $\M$ with $B(a)\in X$. Since $B(a)$ is conflicting with $A(a)$, with respect to $\T$. Then $A(a)$ neither belongs to $X$ nor to $\Cl(X)$. This contradict the fact that $A(a)\in \bigcap_{\A_i\in \M_5}\A_i$. 

ii) $B(a)\notin \M$. Let $Y\subseteq \M$ such that $\KB{\T}{Y} \models B(a)$. Then clearly $\KB{\T}{Y\cup\{A(a)\}}$ is inconsistent. Hence, there exists $D(a)\in \M$ such that $\KB{\T}{\{(D(a),A(a)\}}$ is conflicting and $D(a) \in Y$. This comes down to item (i). Now, since there is no conflict in $\Cl(\M)$ containing $A(a)$, then $A(a)$ belong to all maximally consistent subsets of $\Cl(\M)$, hence $A(a)$ belongs to $\bigcap_{\A_j \in \M_7} \A_j$. Therefore if a $q$ holds in $\langle \M_5, \cap \rangle$, then it holds that $\langle \M_5, \cap\rangle$.  
\end{enumerate}
\end{proof}

\begin{example}[Proof of Figure \ref{sch:safe-inf}, Part 2]
The following counter-examples show that no reciprocal edges hold in Figure \ref{sch:safe-inf}. 
\begin{enumerate}

\item There exists a  \kb{}, and a Boolean query $q$ such that $q$ is a safe conclusion of $\KB{\T}{\M_2}$, but $q$ is not a safe conclusion of $\KB{\T}{\M_1}$: \\
Let us consider $\T$=$\{A \isa B, B \isn C\}$ and $\M=\{\{C(a),A(a),B(a)\}\}$. \\ It is easy to check that $\KB{\T}{\M}$ is inconsistent. We have: \\ 
$\M_1=\{\{C(a)\}\},\{A(a),B(a)\}\}$, and \\
$\M_2=\{\{A(a),B(a)\}\}$.\\
Let $q \gets A(a)$ be a query. One can check that: \\
$\M_2 \models_{\cap} q$, since $\bigcap_{\A_i \in \M_2} \A_i$=$\{A(a),B(a)\}$ but\\
$\M_1\not\models_{\cap}q$.  \\
 
\item There exists a  \kb{}, and a Boolean  query $q$ such that $q$ is a safe conclusion of $\KB{\T}{\M_5}$, but $q$ is not a safe conclusion of $\KB{\T}{\M_1}$: \\
Let us consider $\T$=$\{B \isa D, B \isn C, C\isa D\}$ and $\M=\{\{C(a),B(a)\}\}$. \\ It is easy to check that $\KB{\T}{\M}$ is inconsistent. We have: \\
$\M_1=\{\{C(a)\},\{B(a)\}\}$, and\\
$\M_5=\{\{B(a),D(a)\},\{C(a),D(a)\}\}$. \\
Let $q \gets D(a)$ be a query. One can check that : \\
$\M_5\models_{\cap} q$ since $\bigcap_{\A_i \in \M_5} \A_i$=$\{D(a)\}$, but\\
$\M_1\not \models_{\cap} q$. \\

\item  There exists a  \kb{}, and a Boolean  query $q$ such that $q$ is a safe conclusion of $\KB{\T}{\M_3}$, but $q$ is not a safe conclusion of $\KB{\T}{\M_2}$: \\ 
Let us consider $\T$=$\{B \isn C, C\isa A, B\isa A\}$ and $\M=\{\{C(a),B(a)\}\}$. \\ It is easy to check that $\KB{\T}{\M}$ is inconsistent. We have: \\
$\M_1=\M_2=\{\{C(a)\},\{B(a)\}\}$,and \\
$\M_3=\{\{C(a),A(a)\}\},\{B(a),A(a)\}\}$.\\
Let $q \gets A(a)$ be a query.  One can check that:\\
$\M_3 \models_{\cap} q$,  but  \\
$\M_2 \not\models_{\cap} q$. \\

\item There exists a  \kb{}, and a Boolean  query $q$ such that $q$ is a safe conclusion of $\KB{\T}{\M_4}$, but $q$ is not a safe conclusion of $\KB{\T}{\M_3}$: \\  
Let us consider $\T$=$\{A\isa B$,$B\isn D\}$ and $\M=\{\{A(a),D(a)\}\}$. \\ It is easy to check that $\KB{\T}{\M}$ is inconsistent. We have: \\
$\M_1=\M_2=\{\{A(a)\},\{D(a)\}\}$, \\
$\M_3=\{\{A(a),B(a)\},\{D(a)\}\}$, and \\
$\M_4=\{\{A(a),B(a)\}\}$. \\
Let $q \gets A(a)$ be a query. One can check that \\
$\M_4\models_{\cap} q$ but \\
$\M_3\not\models_{\cap} q$. \\

\item There exists a  \kb{}, and a Boolean   query $q$ such that $q$ is a safe conclusion of $\KB{\T}{\M_3}$, but $q$ is not a safe conclusion of $\KB{\T}{\M_5}$:\\
Let us consider $\T$=$\{A\isa B$, $B\isn D\}$ and $\M=\{\{A(a),D(a),B(a)\}\}$. \\ It is easy to check that $\KB{\T}{\M}$ is inconsistent. We have:\\
$\M_1=\{\{A(a),B(a)\},\{D(a)\}\}$,\\ 
$\M_2=\M_3=\{\{A(a),B(a)\}\}$, and \\
$\M_5=\{\{A(a),B(a)\},\{D(a)\}\}$. \\
Let $q \gets A(a)$ be a query. One can check that \\
$\M_3\models_{\cap} q$ but \\
$\M_5\not\models_{\cap} q$. \\

\item There exists a  \kb{}, and a Boolean  query $q$ such that $q$ is a safe conclusion of $\KB{\T}{\M_6}$, but $q$ is not a safe conclusion of $\M_5$:\\ 
Let us consider $\T$=$\{B\isa C$, $C\isn D\}$ and $\M$=$\{\{B(a),D(a)\}\}$. \\ It is easy to check that $\KB{\T}{\M}$ is inconsistent. We have: \\
$\M_1=\{\{B(a)\},\{D(a)\}\}$, \\
$\M_5=\{\{B(a),C(a)\},\{D(a)\}\}$, and \\
$\M_6=\{\{B(a),C(a)\}\}$. \\ 
Let $q \gets B(a)$ be a query. One can check that\\
$\M_6\models_{\cap} q$ but\\
$\M_5\not\models_{\cap} q$.

\item There exists a   \kb{}, and a Boolean   query $q$ such that $q$ is a safe conclusion of $\KB{\T}{\M_7}$, but $q$ is not a safe conclusion of $\M_5$:\\  
Let $\T$=$\{A\isn B$, $B\isa D\}$ and $\M$=$\{\{A(a), B(a)\}\}$. \\ It is easy to check that $\KB{\T}{\M}$ is inconsistent. We have: \\
$\M_1$=$\{\{A(a)\},\{B(a)\}\}$,  \\
$\M_5$=$\{\{A(a)\},\{B(a),D(a))\}\}$, \\
$\cl{\M}$=$\{\{A(a),B(a),D(a)\}\}$, and\\
$\M_7$=$\{\{A(a),D(a)\},\{B(a),D(a)\}\}$.\\
Let $q \gets D(a)$ be a query.  One can deduce that:\\ 
$\M_7\models_{\cap} q$ but \\
$\M_5\not\models_{\cap} q$. \\

\item There exists a  \kb{}, and a Boolean   query $q$ such that $q$ is a safe conclusion of $\KB{\T}{\M_8}$, but $q$ is not a safe conclusion of $\KB{\T}{\M_7}$:\\  
Let us consider $\T$=$\{A\isa B$, $B\isn C$, $C\isa D\}$ and $\M=\{\{A(a),C(a)\}\}$.  \\ It is easy to check that $\KB{\T}{\M}$ is inconsistent. We have: \\
$\cl{\M}=\{A(a),C(a),B(a),D(a)\}$, \\
$\M_7=\{\{A(a),B(a),D(a)\},\{C(a),D(a)\}\}$, and \\
$\M_8=\{\{A(a),B(a),D(a)\}\}$.\\
Let $q \gets A(a)$ be a Boolean  query. One can deduce that:\\ 
$\M_8\models_{\cap} q$, but\\
$\M_7\not\models_{\cap} q$. \\
\end{enumerate}
\end{example}

\begin{proposition}[Proof of Figure \ref{sch:safe-inf}, Part 3]
Let $\{\circ_1$,...,$\circ_8\}$ be the eight modifier given in Table \ref{tab:abox-modifs}. Let $q$ be a Boolean  query. Then: 
\begin{enumerate}
\item There exists an MBox $\M$ such that the safe inference from $\circ_2(\M)$ is incomparable with the one obtained from $\circ_6(\M)$.
\item There exists an MBox $\M$ such that the safe inference from $\circ_3(\M)$ is incomparable with the one obtained from $\circ_6(\M)$.
\item There exists an MBox $\M$ such that the safe inference from $\circ_6(\M)$ is incomparable with the one obtained from $\circ_7(\M)$.
\item There exists an MBox $\M$ such that the safe inference from $\circ_6(\M)$ is incomparable with the one obtained from $\circ_8(\M)$.
\item There exists an MBox $\M$ such that the safe inference from $\circ_2(\M)$ is incomparable with the one obtained from $\circ_5(\M)$.
\end{enumerate}
\end{proposition}

\begin{example}%[Examples of Figure \ref{sch:safe-inf}] 
The following examples show the incomparabilities stated in the previous proposition. 

\begin{enumerate}

\item The safe inference from $\M_6$ is incomparable with the one obtained from $\M_7$.\\
Let $\T$=$\{C\isa F$,$F\isa A$,$A\isn B$,$B\isa D\}$ and $\M$=$\{\{C(a),B(a)\}\}$.  \\ It is easy to check that $\KB{\T}{\M}$ is inconsistent.  We have: \\
$\M_1=\incl{\M}$=$\{\{C(a)\},\{B(a)\}\}$, and \\
$\M_5=\cl{\M_1}$=$\{\{A(a),C(a),F(a)\}, \{B(a),D(a)\}\}$, and \\
$\M_6=\lex{\M_5}$=$\{\{A(a),C(a),F(a)\}\}$,\\
$\cl{\M}$=$\{\{A(a),C(a),F(a),B(a),D(a)\}\}$, \\
$\M_7=\incl{\cl{\M}}$=$\{\{A(a),C(a),F(a),D(a)\}$, $\{D(a),B(a)\}\}$,\\
Let $q_1 \gets F(a)$ and $q_2\gets D(a)$ be two queries.  One can check that:\\ 
$\langle \M_7$,$\cap\rangle\models q_2$ but  $\langle \M_6$,$\cap\rangle\not\models q_2$\ while \
$\langle \M_6$,$\cap\rangle\models q_1$ but  $\langle \M_7$,$\cap\rangle\not\models q_1$. \\

\item The safe inference from $\M_6$ is incomparable with the one obtained from $\M_8$.\\
Let $\T$=$\{B\isn C$,$B\isa A$,$C\isa A$,$A\isn D$,$D\isa E$,$E\isa F\}$ and $\M$=$\{\{A(a)$,$B(a)$,$C(a)$,$D(a)\}\}$.  \\ It is easy to check that $\KB{\T}{\M}$ is inconsistent.  We have: \\
$\M_1=\incl{\M}$=$\{\{A(a),C(a)\},\{A(a),B(a)\}$, $\{D(a)\}\}$,  \\
$\M_5=\cl{\M_1}$=$\{\{A(a),C(a)\},\{A(a),B(a)\}$, $\{D(a),E(a),F(a)\}\}$,  \\
$\M_6=\lex{\M_5}$=$\{\{D(a),E(a),F(a)\}\}$,\\
$\cl{\M}$=$\{\{A(a), B(a), C(a),D(a),E(a),F(a)\}\}$, \\
$\M_7=\incl{\cl{\M}}$=$\{\{A(a),C(a),E(a),F(a)\}$, $\{A(a),B(a),E(a),F(a)\}$, $\{D(a),E(a),F(a)\}\}$, and \\
$\M_8=\incl{\M_7}$=$\{\{A(a),C(a),E(a),F(a)\}$, $\{A(a),B(a),E(a),F(a)\}\}$ \\
Let $q_1 \gets D(a)$ and $q_2\gets A(a)$ be two queries.  One can check that:\\ 
$\langle \M_8$,$\cap\rangle\models q_2$ but  $\langle \M_6$,$\cap\rangle\not\models q_2$\ while \
$\langle \M_6$,$\cap\rangle\models q_1$ but  $\langle \M_8$,$\cap\rangle\not\models q_1$. \\

\item The safe inference from $\M_2$ is incomparable with the one obtained from $\M_6$.\\
Let $\T$=$\{A\isn B$, $C\isa A$, $B\isa D$, $D\isa F\}$ and $\M$=$\{\{A(a),C(a),B(a)\}\}$.  \\ It is easy to check that $\KB{\T}{\M}$ is inconsistent.  We have: \\
$\M_1=\incl{\M}$=$\{\{A(a),C(a)\},\{B(a)\}\}$, \\
$\M_2=\lex{\M_2}$=$\{\{A(a),C(a)\}\}$, \\
$\M_5=\cl{\M_1}$=$\{\{A(a),C(a)\},\{B(a),D(a),F(a)\}\}$, \\
$\M_6=\lex{\M_5}$=$\{\{B(a),D(a),F(a)\}\}$, \\
Let $q_1 \gets A(a)$ and $q_2\gets B(a)$ be two queries.  One can check that:\\ 
$\langle \M_2$,$\cap\rangle\models q_1$ but  $\langle \M_6$,$\cap\rangle\not\models q_1$\ while \
$\langle \M_6$,$\cap\rangle\models q_2$ but  $\langle \M_2$,$\cap\rangle\not\models q_2$. \\

\item The safe inference from $\M_2$ is incomparable with the one obtained from $\M_5$.\\
Let $\T$=$\{A\isa B$, $C\isa B$, $A\isn C$, $D\isa C\}$ and $\M$=$\{\{A(a),C(a),D(a)\}\}$.  \\ It is easy to check that $\KB{\T}{\M}$ is inconsistent.  We have: \\
$\M_1=\incl{\M}$=$\{\{A(a)\},\{C(a),D(a)\}\}$, \\
$\M_2=\lex{\M_1}$=$\{\{C(a),D(a)\}\}$, and\\
$\M_5=\cl{\M_1}$=$\{\{A(a),B(a)\},\{B(a),D(a),C(a)\}\}$, \\
Let $q_1 \gets D(a)$ and $q_2\gets B(a)$ be two queries.  One can check that:\\ 
$\langle \M_2$,$\cap\rangle\models q_1$ but  $\langle \M_5$,$\cap\rangle\not\models q_1$\ while \
$\langle \M_5$,$\cap\rangle\models q_2$ but  $\langle \M_2$,$\cap\rangle\not\models q_2$. \\
\end{enumerate}
\end{example}

%%%%%%%%%%%%%%%%%%%%%%%%%%%%%%%%%%%%%%%%%%%%%%%%%%%%%%%%%%%%%%%%%%%%%%%%%%%%%%%%%%%%%%% 
% universal semantics
%%%%%%%%%%%%%%%%%%%%%%%%%%%%%%%%%%%%%%%%%%%%%%%%%%%%%%%%%%%%%%%%%%%%%%%%%%%%%%%%%%%%%%%%%%%%%%%%%%%%%%%%%%%%%%%%%%%%

\medskip

\paragraph{Proof of Figure  \ref{sch:all-inf} (universal semantics)} The relation pictured in the figure is proved by the following propositions and examples. 

\begin{proposition}[Proof of Figure \ref{sch:all-inf}, Part 1]
Let $\K_\M$=$\KB{\T}{\M=\{\A\}}$ be an inconsistent  \kb{}. Let $\M_1$,...,$\M_8$ be the MBoxes obtained by applying  the eight modifiers,  given in Table \ref{tab:abox-modifs}, on $\M$. Let $q$ be a Boolean  query. Then: 
\begin{enumerate}
\item $q$ is a universal conclusion of $\KB{\T}{\M_1}$ iff $q$ is a universal conclusion of $\KB{\T}{\M_5}$.
\item $q$ is a universal conclusion of $\KB{\T}{\M_2}$ iff $q$ is a universal conclusion of $\KB{\T}{\M_3}$.
\end{enumerate}
\end{proposition}

\begin{proof}
Item 1 and 2 follow from item 1 of Lemma  \ref{lem:relation-inf2} and the facts that $\M_5=\cl{\M_1}$ and $\M_3=\cl{\M_2}$.
\end{proof}

%%%%%
\begin{proposition}[Proof of Figure \ref{sch:all-inf}, Part 2]
\label{prop:all-inf} 
Let $\K_\M$=$\KB{\T}{\M=\{\A\}}$ be an inconsistent  \kb{}. Let $\M_1$,...,$\M_8$ be the MBoxes obtained by applying  the eight modifiers,  given in Table \ref{tab:abox-modifs}, on $\M$. Let $q$ be a Boolean  query. Then: 
\begin{enumerate}
\item If $q$ is a universal conclusion of $\KB{\T}{\M_1}$ (or $\KB{\T}{\M_5}$) then $q$ is a universal conclusion of $\KB{\T}{\M_2}$. 
\item If $q$ is universal conclusion of $\KB{\T}{\M_3}$ (or $\KB{\T}{\M_2}$) then $q$ is a universal conclusion of $\KB{\T}{\M_4}$. 
\item If $q$ is universal conclusion of $\KB{\T}{\M_1}$ (or $\KB{\T}{\M_5}$) then $q$ is a universal conclusion of $\KB{\T}{\M_6}$. 
\item If $q$ is universal conclusion of $\KB{\T}{\M_7}$ then $q$ is a universal conclusion of $\KB{\T}{\M_8}$. 
\item If $q$ is universal conclusion of $\KB{\T}{\M_1}$ (or $\KB{\T}{\M_5}$) then $q$ is a universal conclusion of $\KB{\T}{\M_7}$.  
\end{enumerate}
\end{proposition}

\begin{proof}  
For Items 1, 2, 3 and 4, we have $\M_2\subseteq \M_1$, $\M_4\subseteq \M_3$, $\M_6\subseteq \M_5$ and $\M_8\subseteq \M_7$.  Then following Item 2 of Lemma \ref{lem:relation-inf}, we have if $\KB{\T}{\M_1}\models_{\forall} q$ then $\KB{\T}{\M_2}\models_{\forall} q$. Similarly for $\M_4\subseteq \M_3$, $\M_6\subseteq \M_5$ and $\M_8\subseteq \M_7$. 

Finally, for item 5 recall first that $\langle \M_5,\forall \rangle \equiv \langle \M_1, \forall \rangle$ and $\forall A \in \M_5, \exists B \in \M_7$ such that $A \subseteq B$. Now let us show that $\forall B \in \M_7, \exists A\in \M_5$ such that $A \subseteq B$. Let $B\in \M_7=\incl{\cl{\M}}$. This means that $B  \subseteq \cl{\M}$ and $B$ is a maximally consistent subset. Let $C\in \incl{\M}$. This means that $C \subseteq \M \subseteq \cl{\M}$. Since $C$ is also a maximally consistent subset then $C \subseteq B$. Now, recall that B is a closed set of assertion, then $A=\Cl(C)\subseteq B$. Therefore we conclude that if a conclusion holds from $\M_5$, then it holds from $\M_7$.
\end{proof}

\begin{example}[Proof of Figure \ref{sch:all-inf}, Part 3]

The following counter-examples show that no reciprocal edges hold in Figure \ref{sch:all-inf}.

\begin{enumerate}

\item There exists a \kb{}, and a Boolean  query $q$ such that $q$ is a universal conclusion of $\KB{\T}{\M_2}$, but $q$ is not a universal conclusion of $\KB{\T}{\M_1}$:\\ 
Let us consider $\T=\{A \isa B, B\isn C\}$ and $\M=\{\{A(a),B(a),C(a)\}\}$. \\ It is easy to check that $\KB{\T}{\M}$ is inconsistent. We have: \\ 
$\M_1=\incl{\M}=\{\{A(a),B(a)\},\{C(a)\}\}$, and \\
$\M_2=\lex{\M_1}=\{\{A(a), B(a)\}\}$. \\
Let $q \gets A(a)$ be a query. One can check that: \\
$\langle \M_2,\forall\rangle\models q$ but \\
$\langle \M_1,\forall\rangle\not\models q$, since $\KB{\T}{\{C(a)\}}\not\models q$. \\
 
\item There exists a  \kb{}, and a Boolean   query $q$ such that $q$ is a universal conclusion of $\KB{\T}{\M_4}$, but $q$ is not a universal conclusion of $\KB{\T}{\M_3}$:\\ 
Let us consider $\T=\{A\isn B, A\isa F\}$ and $\M$=$\{\{A(a), B(a)\}\}$. \\ It is easy to check that $\KB{\T}{\M}$ is inconsistent. We have: \\
$\M_1=\M_2=\{\{A(a)\},\{B(a)\}\}$, \\
$\M_3=\{\{A(a),F(a)\}$,$\{B(a)\}\}$, and \\
$\M_4=\{\{A(a),F(a)\}\}$.\\ 
Let $q \gets F(a)$ be a query. One can check that: \\ 
$\langle\M_4,\forall\rangle\models q$ but \\
$\langle\M_3,\forall\rangle\not\models q$, since $\KB{\T}{\{B(a)\}}\not\models q$. \\

\item There exists a  \kb{}, and a Boolean   query $q$ such that $q$ is a universal conclusion of $\KB{\T}{\M_6}$, but $q$ is not a universal conclusion of $\KB{\T}{\M_5}$:\\  
Let us consider $\T$=$\{B\isa C$, $C\isn D\}$ and $\M=\{\{B(a),D(a)\}\}$. \\ It is easy to check that $\KB{\T}{\M}$ is inconsistent. We have:\\
$\M_1=\{\{B(a)\}, \{D(a)\}\}$, \\  
$\M_5=\{\{B(a),C(a)\},\{D(a)\}\}$, and \\
$\M_6=\{\{B(a),C(a)\}\}$. \\
Let $q \gets C(a)$ be a query. One can check that:\\
$\langle\M_6,\forall\rangle\models q$ but\\
$\langle\M_5,\forall\rangle\not\models q$, since $\KB{\T}{\{D(a)\}}\not\models q$\\

\item There exists a  \kb{}, and a  Boolean   query $q$ such that $q$ is a universal conclusion of $\KB{\T}{\M_8}$, but $q$ is not a universal conclusion of $\KB{\T}{\M_7}$:\\  
Let us consider $\T$=$\{A\isa B$, $B\isn C$, $C\isa D$, $D\isa F\}$ and $\M=\{\{A(a),C(a)\}\}$. \\ It is easy to check that $\KB{\T}{\M}$ is inconsistent. We have: \\
$\cl{\M}=\{A(a),C(a),B(a),D(a),F(a)\}$, \\
$\M_7=\{\{A(a),B(a),D(a),F(a)\},\{C(a),D(a),F(a)\}\}$, and \\
$\M_8=\{\{A(a),B(a),D(a),F(a)\}\}$.\\
Let $q \gets A(a)$ be a query. One can check that:\\ 
$\langle \M_8,\forall\rangle\models q$, but\\
$\langle \M_7,\forall\rangle\not\models q$, since $\KB{\T}{\{C(a),D(a),F(a\}}\not\models q$.\\

\item There exists a  \kb{}, and a Boolean   query $q$ such that $q$ is a universal conclusion of $\KB{\T}{\M_7}$, but $q$ is not a universal conclusion of $\KB{\T}{\M_5}$:\\ 
Let $\T$=$\{A\isn B$, $B\isa D\}$ and $\M$=$\{\{A(a), B(a)\}\}$. \\ It is easy to check that $\KB{\T}{\M}$ is inconsistent. We have: \\
$\M_1$=$\{\{A(a)\},\{B(a)\}\}$,  \\
%$\M_5$=$\{\{A(a)\},\{B(a),D(a)\}\}$, \\
$\cl{\M}$=$\{\{A(a),B(a),D(a)\}\}$, and\\
$\M_7$=$\{\{A(a),D(a)\},\{B(a),D(a)\}\}$.\\
Let $q \gets D(a)$ be a query.  One can check that:\\ 
$\langle \M_7$,$\forall\rangle\models q$ but \\
$\langle \M_1$,$\forall\rangle\not\models q$, since $\KB{\T}{\{A(a)\}}$. \\
\end{enumerate}
\end{example}

\begin{proposition}[Proof of Figure \ref{sch:all-inf}, Part 4]
Let $\{\circ_1$,...,$\circ_8\}$ be the eight modifiers given in Table \ref{tab:abox-modifs}. Then: 
\begin{enumerate}
\item There exists an MBox $\M$ such that the universal inference from $\circ_6(\M)$ is incomparable with the one obtained from $\circ_7(\M)$.
\item There exists an MBox $\M$ such that the universal inference from $\circ_6(\M)$ is incomparable with the one obtained from $\circ_8(\M)$.
\item There exists an MBox $\M$ such that the universal inference from $\circ_2(\M)$ (\resp $\circ_3(\M)$, $\circ_4(\M)$) is incomparable with the one obtained from $\circ_6(\M)$.
\item There exists an MBox $\M$ such that the universal inference from $\circ_2(\M)$ (\resp $\circ_3(\M)$, $\circ_4(\M)$) is incomparable with the one obtained from $\circ_7(\M)$.
\item There exists an MBox $\M$ such that the universal inference from $\circ_2(\M)$ (\resp $\circ_3(\M)$, $\circ_4(\M)$) is incomparable with the one obtained from $\circ_8(\M)$.

\end{enumerate}
\end{proposition}

\begin{example}
The following examples prove the incomparabilities stated in the previous proposition.  
\begin{enumerate}
\item The universal inference from $\M_6$ is incomparable with the one obtained from $\M_7$.\\
Let $\T$=$\{C\isa F$, $F\isa A$, $A\isn B, B\isa D\}$ and $\M$=$\{\{C(a),B(a)\}\}$.  \\ It is easy to check that $\KB{\T}{\M}$ is inconsistent.  We have: \\
$\M_1=\incl{\M}$=$\{\{C(a)\},\{B(a)\}\}$, and \\
$\M_5=\cl{\M_1}$=$\{\{A(a),C(a),F(a)\}, \{B(a),D(a)\}\}$, and \\
$\M_6=\lex{\M_5}$=$\{\{A(a),C(a),F(a)\}\}$,\\
$\cl{\M}$=$\{\{A(a),C(a),F(a),B(a),D(a)\}\}$, \\
$\M_7=\incl{\cl{\M}}$=$\{\{A(a),C(a),F(a),D(a)\}$, $\{D(a),B(a)\}\}$,\\
Let $q_1 \gets F(a)$ and $q_2\gets D(a)$ be two queries.  One can check that:\\ 
$\langle \M_7$,$\forall\rangle\models q_2$ but  $\langle \M_6$,$\forall\rangle\not\models q_2$\ while \
$\langle \M_6$,$\forall\rangle\models q_1$ but  $\langle \M_7$,$\forall\rangle\not\models q_1$. \\

\item The universal inference from $\M_6$ is incomparable with the one obtained from $\M_8$.\\
Let $\T$=$\{B\isn C$, $B\isa A$, $C\isa A, A\isn D, D\isa E, E\isa F\}$ and $\M$=$\{\{A(a), B(a), C(a),D(a)\}\}$.  \\ It is easy to check that $\KB{\T}{\M}$ is inconsistent.  We have: \\
$\M_1=\incl{\M}$=$\{\{A(a),C(a)\}$, $\{A(a),B(a)\}$, $\{D(a)\}\}$, and \\
$\M_5=\cl{\M_1}$=$\{\{A(a),C(a)\}$, $\{A(a),B(a)\}$, $\{D(a),E(a),F(a)\}\}$, and \\
$\M_6=\lex{\M_5}$=$\{\{D(a),E(a),F(a)\}\}$,\\
$\cl{\M}$=$\{\{A(a),B(a),C(a),D(a),E(a),F(a)\}\}$, \\
$\M_7=\incl{\cl{\M}}$=$\{\{A(a),C(a),E(a),F(a)\}$, $\{A(a),B(a),E(a),F(a)\}$, $\{D(a),E(a),F(a)\}\}$, and \\
$\M_8=\lex{\M_7}$=$\{\{A(a),C(a),E(a),F(a)\}$, $\{A(a),B(a),E(a),F(a)\}\}$ \\
Let $q_1 \gets D(a)$ and $q_2\gets A(a)$ be two queries.  One can check that:\\ 
$\langle \M_8$,$\forall\rangle\models q_2$ but  $\langle \M_6$,$\forall\rangle\not\models q_2$\ while \
$\langle \M_6$,$\forall\rangle\models q_1$ but  $\langle \M_8$,$\forall\rangle\not\models q_1$. \\

\item The universal inference from $\M_2$ (\resp{} $\M_3$ and $\M_4$ is incomparable with the one obtained from $\M_6$.\\
Let $\T$=$\{A\isn B$, $C\isa A$, $B\isa D$, $D\isa F\}$ and $\M$=$\{\{A(a),C(a),B(a)\}\}$.  \\ It is easy to check that $\KB{\T}{\M}$ is inconsistent.  We have: \\
$\M_1=\incl{\M}$=$\{\{A(a),C(a)\},\{B(a)\}\}$, \\
$\M_2=\lex{\M_2}$=$\{\{A(a),C(a)\}\}$, \\
$\M_4=\{\{A(a),C(a)\}\}$, \\
$\M_5=\cl{\M_1}$=$\{\{A(a),C(a)\},\{B(a),D(a),F(a)\}\}$, \\
$\M_6=\lex{\M_5}$=$\{\{B(a),D(a),F(a)\}\}$, \\
Let $q_1 \gets A(a)$ and $q_2\gets B(a)$ be two queries.  One can check that:\\ 
$\langle \M_2$,$\forall\rangle\models q_1$ but  $\langle \M_6$,$\forall\rangle\not\models q_1$\ while \
$\langle \M_6$,$\forall\rangle\models q_2$ but  $\langle \M_2$,$\forall\rangle\not\models q_2$. Similarly for $\M_4$ \\

\item The universal inference from $\M_2$ (\resp{} $\M_3$ and $\M_4$ is incomparable with the one obtained from $\M_7$.\\
Let $\T$=$\{A\isn B$, $C\isa A$, $B\isa D$, $D\isa F\}$ and $\M$=$\{\{A(a),C(a),B(a)\}\}$.  \\ It is easy to check that $\KB{\T}{\M}$ is inconsistent.  We have: \\
$\M_1=\incl{\M}$=$\{\{A(a),C(a)\},\{B(a)\}\}$, \\
$\M_2=\lex{\M_2}$=$\{\{A(a),C(a)\}\}$, \\
$\M_4=\{\{A(a),C(a)\}\}$, \\
$\cl{\M}$=$\{\{A(a),C(a),B(a),D(a),F(a)\}\}$, \\
$\M_7=\incl{\cl{\M}}$=$\{\{A(a),C(a),D(a),F(a)\}$, $\{B(a),D(a),F(a)\}\}$, \\
Let $q_1 \gets A(a)$ and $q_2\gets D(a)$ be two queries.  One can check that:\\ 
$\langle \M_2$,$\forall\rangle\models q_1$ but  $\langle \M_7$,$\forall\rangle\not\models q_1$\ while \
$\langle \M_7$,$\forall\rangle\models q_2$ but  $\langle \M_2$,$\forall\rangle\not\models q_2$. Similarly for $\M_4$.\\

\item The universal inference from $\M_2$ (\resp{} $\M_3$ and $\M_4$ is incomparable with the one obtained from $\M_8$.\\
Let $\T$=$\{B\isa A$, $C\isa A, A\isn D, E\isa D, D\isa F\}$ and $\M$=$\{\{A(a),D(a),E(a)\}\}$.  \\ It is easy to check that $\KB{\T}{\M}$ is inconsistent.  We have: \\
$\M_1=\incl{\M}=\{\{A(a)\}$,$\{D(a),E(a)\}\}$, and \\
$\M_2=\lex{\M_1}$=$\{D(a),E(a)\}\}$, and \\
$\M_4$=$\{\{D(a),E(a),F(a)\}\}$,\\
$\cl{\M}$=$\{\{A(a),B(a),C(a),D(a),E(a),F(a)\}\}$, \\
$\M_7=\incl{\cl{\M}}$=$\{\{A(a),C(a),B(a),F(a)\}$, $\{D(a),E(a),F(a)\}\}$, and \\
$\M_8=\incl{\M_7}=\{\{A(a),C(a),B(a),F(a)\}$, \\
Let $q_1 \gets D(a)$ and $q_2\gets A(a)$ be two queries.  One can check that:\\ 
$\langle \M_2$,$\forall\rangle\models q_1$ but  $\langle \M_8$,$\forall\rangle\not\models q_1$\ while \
$\langle \M_8$,$\forall\rangle\models q_2$ but  $\langle \M_2$,$\forall\rangle\not\models q_2$. Similarly for $\M_4$\\
\end{enumerate}
\end{example}

%%%%%%%%%%%%%%%%%%%%%%%%%%%%%%%%%%%%%%%%%%%%%%%%%%%%%%%%%%%%%%%%%%%%%%%%%%%%%%%%%%%%%%%%%%%%%%%%%%%%%%%%%%%%%%%%%%%%%%%%%%%%%%%%%%%%%%%%%%%%%%%%%%%%%%%%%%%%%%%%%%%%%%%%%%%%%%%%%%%%%%%%%%%%%%%%%%%%%%%%%%%%%%%%%%
%%%%%%%%%%%%%%%%%%%%%%%%%%%%%%%%%%%%%%%%%%%%%%%%%%%%%%%%%%%%%%%%%%%%%%%%%%%%%%%%%%%%%%% 
% majority-based semantics
%%%%%%%%%%%%%%%%%%%%%%%%%%%%%%%%%%%%%%%%%%%%%%%%%%%%%%%%%%%%%%%%%%%%%%%%%%%%%%%%%%%%%%%%%%%%%%%%%%%%%%%%%%%%%%%%%%%%

\medskip

\paragraph{Proof of Figure  \ref{sch:maj-inf} (majority-based semantics)} The relation pictured in the figure is proved by the following propositions and examples. 

\begin{proposition}[Proof of Figure \ref{sch:maj-inf}, Part 1]
Let $\K_\M$=$\KB{\T}{\M=\{\A\}}$ be an inconsistent \kb{}. Let $\M_1$,...,$\M_8$ be the MBoxes obtained by applying the eight modifiers, given in Table \ref{tab:abox-modifs}, on $\M$. Let $q$ be a Boolean  query. Then: 
\begin{itemize}
\item $\KB{\T}{\M_1} \models_{maj} q$ iff $\KB{\T}{\M_5} \models_{maj} q$.
\item $\KB{\T}{\M_2} \models_{maj} q$ iff $\KB{\T}{\M_3} \models_{maj} q$.
\item If $\KB{\T}{\M_5} \models_{maj} q$ then $\KB{\T}{\M_7} \models_{maj} q$.
\end{itemize}
\end{proposition}

\begin{proof}
The proof of items 1 and 2 follow immediately from the proof of item 2 of Lemma \ref{lem:relation-inf2}, since $\M_5=\cl{\M_1}$ and $\M_2=\cl{\M_3}$. For Item 3, we have $\forall \A_i \in \M_5, \exists \A_j \in \M_7$ such that $\A_i \subseteq \A_j$. From proof of item 5 of proposition \ref{prop:all-inf}, we have $\forall \A_j \in \M_7, \exists \A_i \in \M_5$ such that $\A_i \subseteq \A_j$. We conclude that if a majority-based conclusion holds from $\M_5$, it holds also from $\M_7$. The converse does not hold.  
\end{proof}

\begin{example}[Proof of Figure \ref{sch:maj-inf}, Part 2]
The following counter-examples show that no reciprocal edges hold in Figure \ref{sch:maj-inf}. 
\begin{enumerate}
\item There exists a  \kb{}, and a query $q$ such that $q$ is a majority-based conclusion of $\KB{\T}{\M_7}$, but $q$ is not a majority-based conclusion of $\KB{\T}{\M_5}$:\\ 
Let $\T$=$\{A\isn B$, $B\isa D\}$ and $\M$=$\{\{A(a), B(a)\}\}$.  \\ It is easy to check that $\KB{\T}{\M}$ is inconsistent.  We have: \\
$\cl{\M}$=$\{\{A(a),B(a),D(a)\}\}$,\\
$\M_7=\incl{\cl{\M}}$=$\{\{A(a),D(a)\},\{B(a),D(a)\}\}$,\\
$\M_1$=$\incl{\M}$=$\{\{A(a)\},\{B(a)\}\}$, and \\
$\M_5=\cl{\M_1}$=$\{\{A(a)\},\{B(a),D(a)\}\}$, \\
Let $q \gets D(a)$  be a query.  One can check that:\\ 
$\langle \M_7$,$ maj \rangle\models q$ but  $\langle \M_5$,$ maj \rangle\not\models q$
\end{enumerate}
\end{example}

\begin{proposition}[Proof of Figure \ref{sch:maj-inf}, Part 3]
Let $\{\circ_1$,...,$\circ_8\}$ be the eight modifiers given in Table \ref{tab:abox-modifs}. Let $q$ be a Boolean  query. Then: 
\begin{itemize}
\item There exists an MBox $\M$ consistent \wrt{} $\T$ such that the majority-based inference from $\KB{\T}{\circ_1(\M)}$ is incomparable with the one obtained from $\KB{\T}{\circ_2(\M)}$.
\item There exists an MBox $\M$ consistent \wrt{} $\T$ such that the majority-based inference from $\KB{\T}{\circ_3(\M)}$ is incomparable with the one obtained from $\KB{\T}{\circ_4(\M)}$. 
\item There exists an MBox $\M$ consistent \wrt{} $\T$ such that the majority-based inference from $\KB{\T}{\circ_5(\M)}$ is incomparable with the one obtained from $\KB{\T}{\circ_6(\M)}$. 
\item There exists an MBox $\M$ consistent \wrt{} $\T$ such that the majority-based inference from $\KB{\T}{\circ_7(\M)}$ is incomparable with the one obtained from $\KB{\T}{\circ_8(\M)}$. 
\end{itemize}
\end{proposition}

\begin{example} The following examples show the incomparabilities stated in the previous proposition. 
\begin{enumerate}

\item The majority-based inference from $\KB{\T}{\M_1}$ is incomparable with the one obtained from $\KB{\T}{\M_2}$. \\
Let $\T$=$\{B\isn C$, $B\isa A$, $C\isa A$, $A\isn D$, $D\isa E$, $E\isa F\}$ and $\M$=$\{\{A(a),B(a),C(a),D(a),E(a),F(a)\}\}$.  \\ It is easy to check that $\KB{\T}{\M}$ is inconsistent.  We have: \\
$\M_1=\incl{\M}$=$\{\{A(a),C(a)\},\{A(a),B(a)\}, \{D(a),E(a),F(a)\}\}$, and \\
$\M_2=\lex{\M_1}$=$\{\{D(a),E(a),F(a)\}\}$ \\
Let $q_1 \gets D(a)$ and $q_2\gets A(a)$ be two queries.  One can check that:\\ 
$\langle \M_1$,$ maj \rangle\models q_2$ but  $\langle \M_2$,$ maj \rangle\not\models q_2$\ while \
$\langle \M_2$,$ maj \rangle\models q_1$ but  $\langle \M_1$,$ maj \rangle\not\models q_1$.

\item The majority-based inference from $\KB{\T}{\M_3}$ is incomparable with the one obtained from $\KB{\T}{\M_4}$. \\
Let $\T$=$\{B\isn C$, $B\isa A$, $C\isa A$, $A\isn D$, $F\isa D$, $D\isa E\}$ and $\M$=$\{\{A(a),B(a),C(a),F(a),D(a)\}\}$.  \\ It is easy to check that $\KB{\T}{\M}$ is inconsistent.  We have: \\
$\M_1=\M_2$=$\{\{A(a),C(a)\},\{A(a),B(a)\}, \{D(a),F(a)\}\}$, \\
$\M_3$=$\{\{A(a),C(a)\},\{A(a),B(a)\}, \{D(a),F(a),E(a)\}\}$, \\
$\M_4$=$\{\{D(a),E(a),F(a)\}\}$ \\
Let $q_1 \gets D(a)$ and $q_2\gets A(a)$ be two queries.  One can check that:\\ 
$\langle \M_3$,$ maj \rangle\models q_2$ but  $\langle \M_4$,$ maj \rangle\not\models q_2$\ while \
$\langle \M_4$,$ maj \rangle\models q_1$ but  $\langle \M_3$,$ maj \rangle\not\models q_1$.

\item The majority-based inference from $\KB{\T}{\M_5}$ is incomparable with the one obtained from $\KB{\T}{\M_6}$. \\
Let $\T$=$\{B\isn C$, $B\isa A$, $C\isa A, A\isn D, F\isa D, D\isa E\}$ and $\M$=$\{\{A(a),B(a),C(a),F(a),D(a),E(a)\}\}$.  \\ It is easy to check that $\KB{\T}{\M}$ is inconsistent.  We have: \\
$\M_1=\M_5$=$\{\{A(a),C(a)\},\{A(a),B(a)\}, \{D(a),F(a),E(a)\}\}$, \\
$\M_6$=$\{\{D(a),E(a),F(a)\}\}$ \\
Let $q_1 \gets D(a)$ and $q_2\gets A(a)$ be two queries.  One can check that:\\ 
$\langle \M_5$,$ maj \rangle\models q_2$ but  $\langle \M_6$,$ maj \rangle\not\models q_2$\ while \
$\langle \M_6$,$ maj \rangle\models q_1$ but  $\langle \M_5$,$ maj \rangle\not\models q_1$.

\item The majority-based inference from $\KB{\T}{\M_7}$ is incomparable with the one obtained from $\KB{\T}{\M_8}$. \\
Let $\T$=$\{B\isn C$, $B\isa A$, $C\isa A, A\isn D, F\isa D, E\isa D\}$ and $\M$=$\{\{A(a),F(a),E(a),B(a),C(a)\}\}$.  \\ It is easy to check that $\KB{\T}{\M}$ is inconsistent.  We have: \\
$\cl{\M}$=$\{\{A(a),C(a),B(a),D(a),F(a),E(a)\}\}$, \\
$\M_7$=$\{\{D(a),E(a),F(a)\},\{A(a),B(a)\}, \{A(a),C(a)\}\}$,and \\
$\M_8$=$\{\{D(a),E(a),F(a)\}\}$,and \\
Let $q_1 \gets D(a)$ and $q_2\gets A(a)$ be two queries.  One can check that:\\ 
$\langle \M_7$,$ maj \rangle\models q_2$ but  $\langle \M_8$,$ maj \rangle\not\models q_2$\ while \
$\langle \M_8$,$ maj \rangle\models q_1$ but  $\langle \M_7$,$ maj \rangle\not\models q_1$.
\end{enumerate}
\end{example}

%%%%%%%%%%%%%%%%%%%%%%%%%%%%%%%%%%%%%%%%%%%%%%%%%%%%%%%%%%%%%%%%%%%%%%%%%%%%%%%%%%%%%%%%%%%%%%%%%%%%%%%%%%%%%%%%%%%%%%%%%%%%%%%%%%%%%%%%%%%%%%%%%%%%%%%%%%%%%%%%%%%%%%%%%%%%%%%%%%%%%%%%%%%%%%%%%%%%%%%%%%%%%%%%%%
%%%%%%%%%%%%%%%%%%%%%%%%%%%%%%%%%%%%%%%%%%%%%%%%%%%%%%%%%%%%%%%%%%%%%%%%%%%%%%%%%%%%%%% 
% existential semantics
%%%%%%%%%%%%%%%%%%%%%%%%%%%%%%%%%%%%%%%%%%%%%%%%%%%%%%%%%%%%%%%%%%%%%%%%%%%%%%%%%%%%%%%%%%%%%%%%%%%%%%%%%%%%%%%%%%%%
\medskip

\paragraph{Proof of Figure  \ref{sch:exist-inf} (existential semantics)} The relation pictured in the figure is proved by the following propositions and examples. 

\begin{proposition}[Proof of Figure \ref{sch:exist-inf}, Part 1]
Let $\K_\M$=$\KB{\T}{\M=\{\A\}}$ be an inconsistent \dllite{} \kb{}. Let $\M_1$,...,$\M_8$ be the eight MBoxes given in Figure \ref{sch:comp-modif} and in Table \ref{tab:abox-modifs}. Let $q$ be a Boolean query. Then: 
\begin{enumerate}

\item if $q$ is an existential conclusion of $\KB{\T}{\M_4}$ then $q$ is an existential conclusion of $\KB{\T}{\M_3}$.

\item $q$ is an existential conclusion of $\KB{\T}{\M_3}$ iff $q$ is an existential conclusion of $\KB{\T}{\M_2}$.  

\item if $q$ is an existential conclusion of $\KB{\T}{\M_2}$ then $q$ is an existential conclusion of $\KB{\T}{\M_1}$. 

\item $q$ is an existential conclusion of $\KB{\T}{\M_1}$ iff $q$ is an existential conclusion of $\KB{\T}{\M_5}$. 

\item if $q$ is an existential conclusion of $\KB{\T}{\M_2}$ then $q$ is an existential conclusion of $\KB{\T}{\M_5}$.

\item if $q$ is an existential conclusion of $\KB{\T}{\M_6}$ the $q$ is an existential conclusion of $\KB{\T}{\M_5}$. 

\item if $q$ is an existential conclusion of $\KB{\T}{\M_8}$ the $q$ is an existential conclusion of $\KB{\T}{\M_7}$. 

\item if $q$ is an existential conclusion of $\KB{\T}{\M_5}$ the $q$ is an existential conclusion of $\KB{\T}{\M_7}$. 

\end{enumerate}
\end{proposition}

\begin{proof}
%We give the proof for $q$ a Boolean  query. The same reasoning holds for instance query.
 Items 1,3, 6, and 7 follow from item 4 of Lemma \ref{lem:relation-inf}.  Items 2 and 4  follow from Item 4 of Lemma \ref{lem:relation-inf2}. Items 5 and 8  hold due the fact that $\circ_2 \subseteq_{cl} \circ_5$ and $\circ_5 \subseteq_{cl} \circ_7$.  
\end{proof}

\begin{example}[Proof of Figure \ref{sch:exist-inf}, Part 2] The following examples show that the reciprocal edges do not hold in  Figure \ref{sch:exist-inf}. We do not  include the examples that prove that all incomparabilities hold, since they are similar.

\begin{enumerate}
\item There exists an existential conclusion of $\KB{\T}{\M_4}$ which is not an existential conclusion of $\KB{\T}{\M_3}$:\\
Let us consider $\T$=$\{A\isa B$,$B\isa C$, $C \isn D$, $D \isa F\}$ and $\M=\{A(a),D(a)\}$. We have:  \\
$\M_1=\M_2=\{\{A(a)\}, \{D(a)\}\}$, \\
$\M_3=\{\{A(a),B(a),C(a)\}, \{D(a),F(a)\}\}$, and \\
$\M_4=\{\{A(a), B(a),D(a)\}\}$. \\
Let $q \gets D(a) \wedge F(a)$ be a Boolean  query.  We have: \\ 
<$\M_3,\exists$>$\models q$, since $\KB{\T}{\{D(a),F(a)\}}\models q$. However <$\M_4,\exists$>$\not\models q$.\\

\item There exists an existential conclusion of $\KB{\T}{\M_1}$ which is not an existential conclusion of $\KB{\T}{\M_2}$:\\ 
Let us consider $\T$=$\{A\sqsubseteq B$,$B\isn C$, $C\isa D\}$ and $\M$=$\{A(a),B(a),C(a),D(a)\}$. We have: \\
$\M_1=\{\{A(a),B(a),D(a)\},\{C(a),D(a)\}\}$, and \\
$\M_2=\{\{A(a),B(a),D(a)\}\}$. \\
Let $q \gets C(a) \wedge D(a)$ be a Boolean  query.  One can easily check that <$\M_1,\exists$>$\models q $ but <$\M_2,\exists$>$\not\models q$.\\

\item There exists an existential conclusion of $\KB{\T}{\M_5}$ which is not an existential conclusion of $\KB{\T}{\M_2}$: \\
Let us consider $\T$=$\{A\isa B$, $B\isa C$, $C\isn D$, $D\isa F\}$ and $\M=\{A(a),B(a),D(a)\}$. We have:\\ 
$\M_1=\{\{A(a),B(a)\},\{D(a)\}\}$, \\
$\M_5=\{\{A(a),B(a),C(a)\},\{D(a),F(a)\}\}$, and \\
$\M_2=\{\{A(a),B(a)\}\}$. \\
Let $q \gets D(a) \wedge F(a)$ be a Boolean  query. One can deduce that : \\
<$\M_5$,$\exists$> $\models q$, but \\ 
<$\M_2$,$\exists$> $\not\models q$. 

\item There exists an existential conclusion of $\KB{\T}{\M_5}$ which is not an existential conclusion of $\KB{\T}{\M_6}$: \\
Let us consider $\T$=$\{A\isa B$, $B\isa C$, $C\isn D$, $D\isa F\}$ and $\M=\{A(a),D(a)\}$.  We have: \\
$\M_1=\{\{A(a)\},\{D(a)\}\}$, \\
$\M_5=\{\{A(a),B(a),C(a)\},\{D(a),F(a)\}\}$, and \\
$\M_6=\{\{A(a),B(a),C(a)\}\}$.\\
Let $q \gets D(a) \wedge F(a)$ be a Boolean  query. One can check  that:\\ 
<$\M_5$,$\exists$> $\models q$, but\\
<$\M_6$,$\exists$> $\not\models q$. 

\item There exists an existential conclusion of $\KB{\T}{\M_7}$ which is not an existential conclusion of $\KB{\T}{\M_8}$: \\
Let us consider $\T$=$\{A\isa B$, $B\isn C$, $C\isa D$, $D\isa F\}$ and $\M=\{A(a),C(a)\}$.  We have: \\
$\cl{\M}=\{A(a),C(a),B(a),D(a),F(a)\}$, \\
$\M_7=\{\{A(a),B(a),D(a),F(a)\},\{C(a),D(a),F(a)\}\}$, and \\
$\M_8=\{\{A(a),B(a),D(a),F(a)\}\}$.\\
Let $q \gets C(a) \wedge D(a)$ be a Boolean  query. One can check  that:\\ 
<$\M_7$,$\exists$> $\models q$, but\\
<$\M_8$,$\exists$> $\not\models q$. 
\end{enumerate}
\end{example}

%%%%%%%%%%%%%%%%%%%%%%%%%%%%%%%%%%%%%%%%%%%%%%%%%%%%%%%%%%%%%%%%%%%%%%%%%%%%%%%%%%%%%%%%%%%%%%%%%%%%%%%%%%%%%%%%%%%%%%%%%%%%%%%%%%%%%%%%%%%%%%%%%%%%%%%%%%%%%%%%%%%%%%%%%%%%%%%%%%%%%%%%%%%%%%%%%%%%%%%%%%%%%%%%%%%%%%%%%%%%%%%%%%%%%%%%%%%%%%%%%%%%%%%%%%%%%%%%%%%%%%%%%%%%%%
%%%%%%% Th 2 %%%%%%%

\paragraph{Theorem  \ref{th:prodsem} [Productivity of semantics] } 
\emph{The inclusion relation $\sqsubseteq$ is the smallest relation that contains
the inclusions $\langle \circ_i , s_k \rangle \sqsubseteq \langle \circ_j , s_k\rangle$  defined by Propositions   \ref{prop:safeprod}-\ref{prop:existprod} and satisfying the two following conditions:
\begin{enumerate}
\item for all $s_j$, $s_p$ and $o_i$, if  $s_j \leq s_p$ then $\langle \circ_i , s_j \rangle \sqsubseteq \langle \circ_i, s_p \rangle$.
\item it is transitive.
\end{enumerate}}

\begin{proof}% [Proof of Theorem \ref{th:prodsem}]
The first point follows from the definition of $\leq$. Let $\K=\KB{\T}{\A}$ and $\K\infer{i}{s_j} q$. This means that $\KB{\T}{\circ_i(\A)}\models_{s_{j}} q$. Since $s_j \leq s_p$, we have  $\KB{\T}{\circ_i(\A)}\models_{s_{p}} q$, hence $\K\infer{i}{s_p} q$. The transitivity of $\sqsubseteq$ follows from its definition. Indeed, consider three semantics $S_1, S_2, S_3$ such that $S_1\sqsubseteq S_2 \sqsubseteq S_3$, then $\forall k, \forall q$, if $\K\models_{S_1} q$ then $\K\models_{S_2} q$ and $\K\models_{S_3} q$. Hence $S_1\sqsubseteq S_3$. 
The following Lemmas \ref{lem:item1th2} and \ref{lem:item2th2} show that there are no other inclusions: the first lemma states that a semantics cannot be included into another semantics with a strictly more cautious inference strategy; the second lemma states that any inclusion from a semantics to another with a strictly less cautious inference strategy can only be obtained by transitivity using an edge ``internal'' to the latter inference strategy. 
\end{proof}

\begin{lemma}
\label{lem:item1th2}
For all $\infr{\circ_i}{s_j}$ and $\infr{\circ_k}{s_p}$,  if $s_p < s_j$ then  $\infr{\circ_i}{s_j} \not \sqsubseteq \infr{\circ_k}{s_p}$;
\end{lemma}

\begin{proof}
The prove this lemma, we consider the following example. Let $\K$=$\KB{\T}{\A}$ with \\
$\A$=$\{p_a(f, a)$,$p_b(f, b)$, $p_c(f, c)$, $p_d(f, d)$, $p_e(f, e)\}$ and \\ $\T$=$\{p_a(Z, X)$,$p_b(Z, Y)$$\rightarrow$$\bot$; $p_b(Z, X),p_c(Z,
Y)$$\rightarrow$$\bot$; $p_b(Z, X),p_d(Z, Y)$$\rightarrow$$\bot$; $p_c(Z, X),p_d(Z,
Y)$$\rightarrow$$\bot$, $p_a(Z, X),p_e(Z, Y)$$\rightarrow$$\bot\}$. We have
$\circ_1(\A)$
%($R(\A)$) are
contains $\{p_a(f, a)$, $p_c(f, c)\}$, $\{p_a(f, a)$,
$p_d(f, d)\}$, $\{p_b(f, b)$, $p_e(f, e)\}$, $\{p_e(f, e)$, $p_d(f, d)\}$ and
$\{p_e(f, e)$, $p_c(f, c)\}$. Since $\T$ contains only negative
constraints and all ABoxes in $\circ_1(\A)$ have the same size,  then we have $\circ_{1}(\A)=
\circ_{2}(\A)=\circ_{3}(\A) = \ldots = \circ_{8}(\A)$.

One can check that $\circ_1(\A) \models_\exists p_a(f, a)$, but $\circ_1(\A)
\not\models_X p_a(f, a)$, for $X \in \{maj, \forall, \cap\}$, thus there is
no $\infr{\circ_i}{\exists} \sqsubseteq \infr{\circ_k}{s_p}$ for $s_p \in \{maj, \forall, \cap\}$. Similarly, we have $\circ_{A} \models_{maj} p_e(f, e)$, but $\circ_1(\A) \not\models_X p_e(f, e)$, for $X \in \{\forall, \cap\}$, thus there is no $\infr{\circ_i}{maj} \sqsubseteq \infr{\circ_k}{s_p}$ for $s_p \in \{\forall, \cap\}$.

Finally, by adding to the previous examples the five following rules:
$p_a(X, Y) \rightarrow p(X,Z)$, \dots, $p_e(X, Y) \rightarrow p(X, Z)$ (which  produce non-ground
atoms), we do not change the repairs, hence we still have the property $\circ_{1}(\A)=
\circ_{2}(\A)=\circ_{3}(\A) = \ldots = \circ_{8}(\A)$. Furthermore,  we
have $\circ_1(A) \models_\forall \exists X \exists Y p(X, Y)$, but
$\circ_1(A) \not\models_\cap \exists X \exists Y p(X, Y)$, thus there is
no $\infr{\circ_i}{\forall} \sqsubseteq \infr{\circ_k}{\cap}$.
\end{proof}

\begin{lemma}
\label{lem:item2th2}
For all $\infr{\circ_i}{s_j}$ and $\infr{\circ_k}{s_p}$,  if $\infr{\circ_i}{s_j} \sqsubseteq \infr{\circ_k}{s_p}$ and $s_j < s_p$, then $\infr{\circ_i}{s_p} \sqsubseteq \infr{\circ_k}{s_p}$.
\end{lemma}

To prove this lemma, we did not find a ``generic'' example as in the previous proof, hence we checked all cases one by one. Examples showing the incomparability can easily be found (similarly to what has been done for the proofs in the preceding section). 

\medskip
Note that when we restrict queries to ground atoms additional inclusions hold. We did not consider this specific, nevertheless important, case in the paper for space restriction reasons. 

Finally, the following schema pictures all inclusions between semantics.

\end{document}